\documentclass[lettersize,journal]{IEEEtran}
\usepackage{setspace}
\usepackage{array}
\usepackage[caption=false,font=normalsize,labelfont=sf,textfont=sf]{subfig}
\usepackage{stfloats}
\usepackage{url}
\usepackage{verbatim}
\usepackage{bm}
\usepackage{times}
\usepackage{float}
\usepackage{amsthm}

\newtheorem{assumption}{Assumption}
\newtheorem{proposition}{Proposition}
\newtheorem{lemma}{Lemma}

\newtheorem{remark}{Remark}
\newtheorem{theorem}{Theorem}

\IEEEoverridecommandlockouts
\usepackage{cite}
\usepackage{amsmath,amssymb,amsfonts}
\allowdisplaybreaks[4]
\usepackage[noend]{algorithmic}
\usepackage{graphicx}
\usepackage{textcomp}
\usepackage{xcolor}
\usepackage{booktabs}
\usepackage{multirow}
\usepackage{subfig}
\usepackage{epstopdf}
\usepackage{epsfig}

\usepackage{stmaryrd}
\usepackage{mathrsfs}
\usepackage{color}
\usepackage{colortbl}  

\usepackage[mathscr]{eucal}

\usepackage{algorithm}
\usepackage{algorithmic}
\renewcommand{\algorithmicrequire}{ \textbf{Input:}} 
\renewcommand{\algorithmicensure}{ \textbf{Output:}} 

\begin{document}
\title{Tensor Decomposition based Personalized Federated Learning}

\author{Qing~Wang,~\IEEEmembership{Member,~IEEE,}
        ~Jing Jin,
        ~Xiaofeng Liu,
        ~Huixuan~Zong$^*$,
        ~Yunfeng Shao,~\IEEEmembership{Member,~IEEE}
        ~and~Yinchuan Li$^{*}$,~\IEEEmembership{Member,~IEEE}
\thanks{Qing Wang, Jing Jin, Huixuan Zong and Xiaofeng Liu are with School of Electrical and Information Engineering, Tianjin University, Tianjin, China (e-mail: wangq@tju.edu.cn, jingjinyx@tju.edu.cn, \textcolor{black}{2019234080@tju.edu.cn}, xiaofengliull@tju.edu.cn)}
\thanks{Yunfeng Shao and Yinchuan Li are with Huawei Noah's Ark Lab, Beijing, China (e-mail: shaoyunfeng@huawei.com, liyinchuan@huawei.com)}
\thanks{This work is sponsored by the National Natural Science Foundation of China under Grant No. 61871282.}
\thanks{$^*$ Corresponding author: Huixuan Zong, Yinchuan Li.}
}


\maketitle

\begin{abstract}
Federated learning (FL) is a new distributed machine learning framework that can achieve reliably collaborative training without collecting users' private data. 
However, due to FL's frequent communication and average aggregation strategy, they experience challenges scaling to statistical diversity data and large-scale models.
In this paper, we propose a personalized FL  framework, named Tensor Decomposition based Personalized Federated learning (TDPFed), in which we design a novel tensorized local model with tensorized linear layers and convolutional layers to reduce the communication cost.
TDPFed uses a bi-level loss function to decouple personalized model optimization from the global model learning by controlling the gap between the personalized model and the tensorized local model.
Moreover, an effective distributed learning strategy and two different model aggregation strategies are well designed for the proposed TDPFed framework. Theoretical convergence analysis and thorough experiments demonstrate that our proposed TDPFed framework achieves state-of-the-art performance while reducing the communication cost.
\end{abstract}


\IEEEpeerreviewmaketitle

\section{Introduction}

\IEEEPARstart{D}{eep} neural networks often rely on tremendous amounts of training data~\cite{lecun2015deep}. The increasing demand for data sharing and integration usually contributes to privacy leaks\cite{1-Communication}, which considerably undermines the deployment of deep learning, especially in security-critical application domains, like Beyond 5G wireless networks and autonomous vehicles. Federated Learning (FL), proposed by Google in 2016~\cite{1-Communication}, can achieve reliably collaborative training without collecting users' private data. In traditional federated learning, clients train local models individually without sharing personal data, and the server periodically collects the client's local model to generate a global model for joint training. After that, the new global model is distributed to all clients and replaces their current local models. However, due to traditional FL's frequent communication and average aggregation strategy, they experience challenges scaling to statistical diversity data and large-scale models.

One of the main challenges of FL is the statistical diversity training data, i.e., non-independent and identical distribution (non-IID) data since the data distributions among clients are distinct~\cite{9712214,zhang2022personalized,li2022federated}. These statistically disparate data from different users make traditional FL methods divergent. Recently, personalized FL methods have been proposed to address this problem~\cite{3-2020Personalized, 9743558, 9712214}. For example, pFedMe\cite{3-2020Personalized} uses Moreau envelopes to decouple personalized model optimization from the global model learning in a bi-level problem stylized for personalized FL. 
Another major challenge is excessive traffic volume in FL, especially in training large-scale models. Tensor decomposition~\cite{Nicholas2017TensorDF,AndrzejCichocki2015TensorDF,EvangelosEPapalexakis2016TensorsFD,NikosKargas2018TensorsLA} provides a low-rank representation of parameter matrices, significantly reducing the number of parameters. For example, CANDECOMP/PARAFAC (CP) decomposition \cite{4-Tensor} decomposes a tensor into the sum of several rank-1 tensors, which can effectively reduce the model size and speed up the training process.

In this work, we propose a novel Tensor Decomposition based Personalized Federated learning (TDPFed) framework, extracting vital information from models and reducing communication volume. Through tensor decomposition of the local model and global model, users can train the tensorized local model and the personalized model alternatively to achieve model compression and feature extraction. The main contributions of this paper are as follows:

1) To reduce the communication costs when aggregating distributed non-IID data, we first build the TDPFed model. On the client-side, a personalized model is trained using the local data. Then a new tensorized local model with low dimension factor matrices is designed to reduce the communication cost. The server aggregates the uploaded tensorized local model and then broadcasts both the full tensor and the factor matrices to the client, updating the personalized model and the tensorized local model, respectively. 

2) We design a bi-level loss function, which controls the gap between the tensorized local model and the personalized model. At the inner level, each client aims to obtain an optimal personalized model using its own data and is maintained at a bounded distance from the local model. At the outer level, the tensorized local model of each client is updated using stochastic gradient descent with respect to multiple-client data.  

3) To well express the tensorized local model, we define the tensorized linear layers and convolutional layers, which are the components of our proposed tensorized neural networks, including CP-DNN, CP-VGG, etc. Then we design a distributed learning strategy for our proposed TDPFed framework. More elaborately, we decompose the local model into factor matrices for low-rank representations by training. The client uses the factor matrix as a local model, interacts with the server via iterative training factor matrix, and obtains the generalization ability during the aggregation process, so as to provide a reference for personalized model training. Furthermore, we design two different model aggregation strategies, aggregating composed tensor and aggregating factor matrices. Among them, the aggregating factor matrix is not only faster to train but also more robust to the model. 

4) We present the theoretical convergence analysis of TDPFed. It shows that the TDPFed’s convergence rate is state-of-the-art with linear speedup. Extensive experimental results show that FedMac outperforms various state-of-the-art personalization algorithms.

The remainder of this paper is organized as follows. In Section II, related works are reviewed. Section III presents the preliminaries and system model. The proposed TDPFed framework is given in Section IV, then the training strategy of TDPFed is proposed in Section V. Section VI presents the convergence analysis results. Experimental results are given in Section VII. Finally, Section VIII concludes the paper.


\section{Related Works}


Google's McMahan et al. first proposed the concept of Federated learning and the Federated Average (FedAvg) algorithm based on Stochastic Gradient Descent (SGD) \cite{1-Communication}. The basic federated learning framework consists of users and a central server. Users train with the local data and only upload the model to the server for aggregation to obtain the global model, which realizes multi-party cooperative machine learning through the iterative update between users and the server under local data protection.


To make the FL algorithm robust on non-IID data, different personalized FL (PFL) methods have been proposed. Specifically, Mansour et al. proposed three algorithms: user clustering, data interpolation, and model interpolation \cite{8-Three}. An intermediate model is added between the local model and the global model to group users with similar data distributions, and each group trains a model.
In a text prediction model based on the Google keyboard \cite{10-Federated}, Wang et al. proposed that some or all parameters of the trained global model are retrained on the local data \cite{9-Federated}. Users train the personalized model and test the personalized model and global model, calculate and upload the changes of test indexes, and determine the setting of hyperparameters through the evaluation of the personalized model.
Smith et al. proposed a MOCHA framework \cite{11-Federated}, which combined multi-task learning with federated learning objective function, optimized federated learning model and task relation matrix, corresponding user's local training to multiple tasks, fitting independent but related models, and theoretically derived the convergence performance of the algorithm. However, this multi-task approach has limited ability to extend large-scale networks, and the algorithm is limited to the case where the objective function is convex.
Finn et al. proposed the Model-Agnostic Meta-Learning (MAML) algorithm \cite{12-Model-Agnostic}. Model-agnostic means that the algorithm is suitable for any model trained on gradient descent, such as classification, regression, policy gradient reinforcement learning, etc. The algorithm trains a meta-model as the initial point of a new task so that the model can be trained quickly on new tasks. Fallah et al. proposed a personalized variant of the federated average algorithm based on the MAML algorithm, called per-FedAvg \cite{13-Meta-Learning}, with the purpose of finding a metamodel among users as the initial model, while the final model trained by each user is adapted to local data. The algorithm improves the user's objective function. After the user conducts gradient descent according to the model's prediction loss in the local data set, the user trains and aggregates the initial local model according to the objective function to find the initial model shared by all users.

Recently, the state-of-the-art convergence rate and good performance on non-IID data make Personalized Federated Learning with Moreau Envelopes (pFedMe)~\cite{3-2020Personalized}  popular among PFL methods. Using Moreau envelopes as clients' regularized loss,  pFedMe decouples personalized model optimization from the global model learning in a bi-level problem. However, the pFedMe algorithm cannot solve the problem of the enormous traffic between the user and the server in FL. 

Moreover, some recent works achieve parameter sharing and extract model feature information by applying tensor decomposition to the FL network model and using algebraic tensor operations to improve model performance. Chen et al. performed BTD decomposition of convolution kernel tensor from the perspective of BTD decomposition through derivation of convolution operation process \cite{23-Sharing}. 
Ma et al. proposed a new self-attention method based on the ideas of parameter sharing and low-rank approximation \cite{25-A}. 
Bulat et al. used the Tucker decomposition structure to re-parameterize the weight tensors in CNN \cite{26-Incremental}. 
Mai et al. realized multi-mode fusion through tensor fusion and proposed Graph Fusion Network (GFN) to fuse the coded representation of all modes \cite{27-Modality}.
In this paper, we propose a novel tensor decomposition based personalized federated learning framework.
Our personalized structure can effectively reduce the federated learning communication traffic while showing robustness on non-IID data.

\section{Preliminaries \& System Model}

\subsection{Preliminaries}
 We introduce the notation used in this paper as follows. The tensor is a multidimensional array, we denote it by Euler script letter, the $i$-th element of vector $\boldsymbol{a}$ is denoted as $\boldsymbol{a}_i$, element $(i,j)$ of matrix $\bold{A}$ is denoted as $\boldsymbol{a}_{ij}$, element $(i,j,k)$ of a third-order tensor $\mathcal{X}$ is denoted as $\boldsymbol{x}_{ijk}$. The $n$-th element in the series is expressed in the form of superscript with parentheses, for example, $\bold{A}^{(n)}$ represents the $n$-th matrix in the series.
 
 For matrices $\bold{A}$ and $\bold{B}$, Kronecker product is denoted as $\bold{A}\otimes\bold{B}$, Khatri-Rao product is denoted as $\bold{A}\odot\bold{B}$, Hadamard product is denoted as $\bold{A}*\bold{B}$, which is elemental product of two matrices. The frobenius norm of tensor $\mathcal{X}\in\mathbb{R}^{I_1\times I_2\times\dots\times I_N}$ is defined as $\|\mathcal{X}\|_2=\sqrt{\sum_{i_1=1}^{I_1}\sum_{i_2=1}^{I_2}\dots\sum_{i_N=1}^{I_N}{\boldsymbol{x}}^2_{i_{1}i_{2}\dots i_{N}}}$. Inner product of two tensors of the same size is defined as $\left\langle\mathcal{X},\mathcal{Y}\right\rangle=\sum_{i_1=1}^{I_1}\sum_{i_2=1}^{I_2}\dots\sum_{i_N=1}^{I_N}{\boldsymbol{x}}_{i_{1}i_{2}\dots i_{N}}{\boldsymbol{y}}_{i_{1}i_{2}\dots i_{N}}$. We use $\circ$ to represent the outer product of vectors, if an $N$-dimensional tensor can be expressed as the outer product of $N$ vectors, the rank of this tensor is $1$. 
 
 Unfolding is the process of reordering an $N$-dimensional tensor into matrices, which arranges mode-$n$ fibers of the tensor into columns of the matrix. For the tensor $\mathcal{X}\in\mathbb{R}^{I_1\times I_2\times\dots\times I_N}$, we write its mode-$n$ unfolding as $\bold{X}_{(n)}$. Mode-$n$ product is the multiplication of tensor and matrix in the $n$-th dimension, we denote the mode-$n$ product of the tensor $\mathcal{X}\in\mathbb{R}^{I_1\times I_2\times\dots\times I_N}$ and the matrix $\bold{U}\in\mathbb{R}^{J\times I_n}$ as $\mathcal{X}\times_n\bold{U}$. Mode-$n$ product can be expressed in the form of tensor unfolding as $\mathcal{Y}=\mathcal{X}\times_n\bold{U}\leftrightarrow\bold{Y}_{(n)}=\bold{U}\bold{X}_{(n)}$. Mode-$n$ product is the same for the multiplication of tensor and vector, the result is an $N-1$ dimensional tensor. A tensor can be multiplied by multiple matrices or vectors, for example, for a sequence of vectors $\boldsymbol{v}^{(n)}\in\mathbb{R}^{I_n}, n=1,\dots,N$, the product of all modes is a scalar
 \begin{multline}
 \mathcal{X}\times_1 \boldsymbol{v}^{(1)}\times_2 \boldsymbol{v}^{(2)}\times\cdots\times_N \boldsymbol{v}^{(N)}\\
=\sum_{i_1=1}^{I_1}\sum_{i_2=1}^{I_2}\dots\sum_{i_N=1}^{I_N}\boldsymbol{x}_{i_1 i_2\cdots i_N}\boldsymbol{v}_{i_1}^{(1)}\boldsymbol{v}_{i_2}^{(2)}\cdots \boldsymbol{v}_{i_N}^{(N)}. \nonumber
 \end{multline}
 
 CP decomposition is a method of low-rank estimation of tensor, which decomposes the tensor into the sum of rank-$1$ tensors. For an $N$-dimensional tensor $\mathcal{X}\in\mathbb{R}^{I_1\times I_2\times\dots\times I_N}$ and a positive integer $R$, we hope to express it in the form as
\begin{equation*}
\mathcal{X}\approx\sum_{r=1}^{R}\boldsymbol{a}_{r}^{(1)}\circ\dots\circ\boldsymbol{a}_{r}^{(N)},
\end{equation*}
where the vector $\boldsymbol{a}_{r}^{(n)}\in\mathbb{R}^{I_n}, n=1,\dots, N, r=1,\dots, R$, and the minimal possible $R$ is called as CP rank. These vectors form the factor matrix $\bold{A}^{(n)}=\left[\boldsymbol{a}_{1}^{(n)}\dots\boldsymbol{a}_{R}^{(n)}\right]$. Using Kruskal operator, CP decomposition can be written as
\begin{equation*}
\mathcal{X}
\approx\llbracket \bold{A}^{(1)},\bold{A}^{(2)},\dots,\bold{A}^{(N)}\rrbracket
=\sum_{r=1}^{R}\boldsymbol{a}_{r}^{(1)}\circ\dots\circ\boldsymbol{a}_{r}^{(N)}.
\end{equation*}

Moreover, the product can be written in matrix form as
\begin{multline}
    \llbracket \bold{A}^{(1)},\bold{A}^{(2)},\dots,\bold{A}^{(N)}\rrbracket_{(n)}\\
=\bold{A}^{(n)}(\bold{A}^{(N)}\odot\dots\odot\bold{A}^{(n+1)}\odot\bold{A}^{(n-1)}\odot\dots\odot\bold{A}^{(1)}). \nonumber
\end{multline}




\subsection{System Model}
Consider a tensorized personalized federated learning framework consisting of $K$ clients and a central server as shown in Fig. \ref{fig:system model}. Each client trains two kinds of models, namely personalized model and tensorized local model, defined as follows: 

\begin{figure}[ht]
	\centering\includegraphics[width=0.5\textwidth]{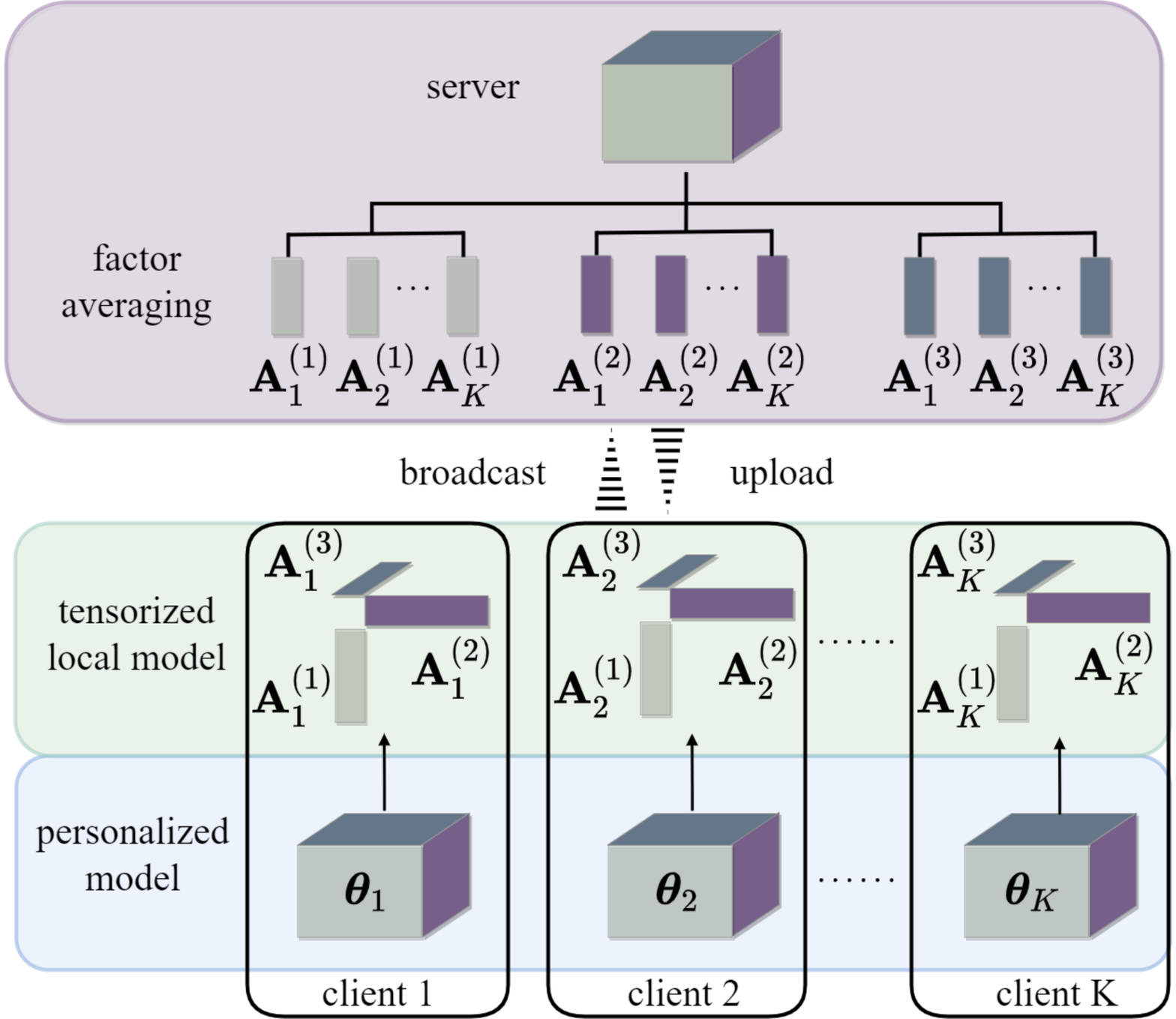}
	\caption{TDPFed System Model.}
	\label{fig:system model}
\end{figure}

\begin{itemize}
\item \textbf{Personalized model}. Each client trains a personalized model using its local dataset distribution and keeps it locally.
Specifically, an approximate solution $\widetilde{\boldsymbol{\theta}}_k$ is obtained using a mini-batch randomly sampled data.
\end{itemize}

\begin{itemize}
\item \textbf{Tensorized local model}. We decompose the $N$-dimensional weighting parameter $\boldsymbol{w}$ of local model into the factor matrix $\{\bold{A}^{(1)},\dots,\bold{A}^{(N)}\}$. The tensorized local model approaches the personalized model, and interacts with the server to aggregate the global model.
\end{itemize}

For cross-device federated learning, client devices such as smartphones, IoT devices, etc., often have limited storage and computing resources and cannot train large models. To allow server-side and the client-side store models with different structures, clients participate in aggregation on the server-side while training a personalized model and a tensorized local model by themselves. 
To emphasize, the aim to tensorize the local model is to compress the high-dimensional weight parameters into low-dimensional subspaces, reducing the communication cost in federated learning.


\section{Proposed TDPFed Framework}

In this section, we propose our tensor decomposition based personalized federated learning framework. We design a bi-level objective function for personalized training and propose a tensorized local model for communication efficiency.

\subsection{Overall Objective Function}
Denote the original high-dimensional weight parameters $\mathcal{W} \in {\mathbb{R}^{{I_1} \times {I_2} \times \dots \times {I_N}}}$ as a tensor of order $N$, which can be transformed into the CP decomposition form with rank $R$, i.e.,
${\cal W} \approx [\kern-0.15em[ {{\boldsymbol{\mathbf{A}}^{\left( 1 \right)}}, \dots, {\boldsymbol{\mathbf{A}}^{\left( N \right)}}} 
 ]\kern-0.15em], $
where the $r$-th column vector ${\boldsymbol{a}}_r^{\left( n \right)}$ of factor matrix ${\boldsymbol{\mathbf{A}}^{\left( n \right)}} \in {\mathbb{R}^{{I_n} \times R}},n = 1,\dots,N$, ${{\boldsymbol{\mathbf{A}}}^{(n)}} = [ {\boldsymbol{a}_1^{\left( n \right)}, \ldots ,\boldsymbol{a}_R^{\left( n \right)}} ]$. Then, the objective function of the proposed TDPFed framework is defined as  
\begin{multline}
\label{eq:Loss}
F\left( {{{\mathbf{A}}^{\left( 1 \right)}}, \ldots,{{\mathbf{A}}^{\left( N \right)}}} \right) \triangleq \\
\mathop {\min }\limits_{{{\mathbf{A}}_k^{(n)}},n = 1,...,N} 
\left\{ {\sum\limits_{k = 1}^K {\frac{{\left| {{\mathcal{D}_k}} \right|}}{{\left| \mathcal{D} \right|}}} {F_k}\left( {{{\mathbf{A}}_k^{\left( 1 \right)}}, \ldots ,{{\mathbf{A}}_k^{\left( N \right)}}} \right)} \right\},
\end{multline}
where $\mathcal{D}_k=\{\boldsymbol{x}_i,\boldsymbol{y}_i\}_{i=1}^{|\mathcal{D}_k|}$ is the local dataset of client $k$ with $|\mathcal{D}_k|$ representing number of local samples; $\boldsymbol{x}_i$ is the input of the sample and $\boldsymbol{y}_i$ is the label;  ${F_k}$ is the local objective function of client $k$ defined as
\begin{multline}
    {F_k}\left( {{{\mathbf{A}}_k^{\left( 1 \right)}}, \ldots ,{{\mathbf{A}}_k^{\left( N \right)}}} \right) \triangleq \mathop {\min }\limits_{{\boldsymbol{\theta} _k} \in {\mathbb{R}^{{I_1} \times {I_2} \times \dots \times {I_N}}}} \\
\left\{ {{f_k}\left( {{\boldsymbol{\theta} _k}} \right) + \frac{\lambda }{2}{{\left\| {{\boldsymbol{\theta} _k} - \left[\kern-0.15em\left[ {{\mathbf{A}_k^{\left( 1 \right)}},\dots,{{\mathbf{A}}_k^{\left( N \right)}}} 
 \right]\kern-0.15em\right]} \right\|}^2}} \right\},
 \label{eq:clientloss}
\end{multline}
where ${\boldsymbol{\theta} _k} $ is the personalized model of client $k$, ${f_k}\left( {{\boldsymbol{\theta} _k}} \right)$ is the expectation of loss prediction of the personalized model on the local data distribution of client $k$, and $[\kern-0.15em[ {{\boldsymbol{\mathbf{A}}_k^{\left( 1 \right)}}, \dots, {\boldsymbol{\mathbf{A}}_k^{\left( N \right)}}} 
 ]\kern-0.15em]$ is the tensorized local model. Note that ${F_k(\cdot)}$ includes an ${\ell_2}$ norm regularization term, which is used to control the distance between the personalized model and the tensorized local model, and $\lambda$ is used to control the degree of regularization.  

\subsection{Tensorized Local Model}\label{sec:TLM}

Usually, a local model consists of multiple layers, such as fully connected layers and convolutional layers, etc.
In fact, the weight of each layer is a matrix or high-dimensional tensor. To tensorize the local model, we first give the tensor decomposition representation of the fully connected layer and convolutional layer.  In the expressions below, we use lowercase letters and subscripts to denote the indices of elements and dimensions, respectively, unless otherwise stated.

\subsubsection{Tensorized Fully Connected Layer} 

In the fully connected layer, the dimension of the weight matrix is related to the dimension of the input and output vector. The output vector $\boldsymbol{y} \in {{\mathbb R}^{{I_{{\rm{out}}}}}}$ of the fully connected layer is equal to 
 \begin{equation}
 \begin{aligned}
\boldsymbol{y} = \boldsymbol{W}\boldsymbol{x} + \boldsymbol{b},
 \end{aligned}
 \end{equation} 
where $\boldsymbol{x} \in {{\mathbb R}^{{I_{\rm{in}}}}}$, $\boldsymbol{W} \in {{\mathbb R}^{{I_{{\rm{out}}}} \times {I_{\rm{in}}}}}$, and $\boldsymbol{b} \in {{\mathbb R}^{{I_{{\rm{out}}}}}}$ are the input vector, the weight matrix, and the offset, respectively.  
 
Denote $\boldsymbol{W} \approx [\kern-0.15em[ {{\boldsymbol{\mathbf{A}}^{\left( 1 \right)}},{\boldsymbol{\mathbf{A}}^{\left( 2 \right)}}} 
 ]\kern-0.15em] = \sum_{r = 1}^R {\boldsymbol{a}_r^{\left( 1 \right)} \circ \boldsymbol{a}_r^{\left( 2 \right)}} $, among them, ${\boldsymbol{w}_{i,j}} = \sum_{r = 1}^R {\boldsymbol{a}_{r,i}^{\left( 1 \right)}\boldsymbol{a}_{r,j}^{\left( 2 \right)}} $, as can be seen, the factorization of a weight matrix is rewritten as a matrix product of factors
 \begin{equation}
 \begin{aligned}
\boldsymbol{y} \approx {\boldsymbol{\mathbf{A}}^{\left( 1 \right)}}{\boldsymbol{\mathbf{A}}^{\left( 2 \right)T}}\boldsymbol{x} + \boldsymbol{b}. 
 \end{aligned}
 \end{equation} 

\subsubsection{Tensorized Convolutional Layer}

In the convolution layer, the weight is a stack of a series of convolution kernels. The convolution kernel includes two spatial dimensions and one channel dimension. In batch sample training, these convolution kernels are organized as a $4$-th order tensor, represented by ${\cal W} \in {{\mathbb R}^{{I_d} \times {I_d} \times {I_S} \times {I_T}}}$, where the width and height of the convolution window are ${I_d}$, while ${I_S}$, ${I_T}$ are the number of input and output channels.

The third-order input and output tensors are ${\cal X} \in {{\mathbb R}^{I_P \times I_Q \times {I_S}}}$ and  ${\cal Y} \in \mathbb{R}^{\left( {I_P - {I_d} + 1} \right) \times \left( {I_Q - {I_d} + 1} \right) \times {I_T}}$, respectively. The convolution operation is a linear mapping from the input tensor to the output tensor as
\begin{equation}
\begin{aligned}
\label{output of convolutional layer}
{\boldsymbol{y}_{i,j,t}} = \sum\limits_{p = i- \delta }^{i + \delta } {\sum\limits_{q = j - \delta }^{j + \delta } {\sum\limits_{s = 1}^{{I_S}} {{\boldsymbol{w}_{\left( {p - i + \delta ,q - j + \delta ,s,t} \right)}}{\boldsymbol{x}_{p,q,s}}} }},
\end{aligned}
\end{equation} 
where $\delta=({{I_d} - 1} )/2 $.
 Performing rank-$R$ CP decomposition, we have ${\cal W} \approx \left[\kern-0.15em\left[ {{\mathbf{A}^{\left( 1 \right)}},{\mathbf{A}^{\left( 2 \right)}},{\mathbf{A}^{\left( 3 \right)}},{\mathbf{A}^{\left( 4 \right)}}} 
 \right]\kern-0.15em\right]$ with ${\mathbf{A}^{\left( 1 \right)}} \in {{\mathbb R}^{{I_d} \times R}}$, ${\mathbf{A}^{\left( 2 \right)}} \in {{\mathbb R}^{{I_d} \times R}}$, ${\mathbf{A}^{\left( 3 \right)}} \in {{\mathbb R}^{{I_S} \times R}}$ and ${\mathbf{A}^{\left( 4 \right)}} \in {{\mathbb R}^{{I_T} \times R}}$, then the parameters in (\ref{output of convolutional layer}) can be expressed as
\begin{equation}
\begin{aligned}
\label{Convolutional layer parameters}
{\boldsymbol{w}_{p - i + \delta ,q - j + \delta ,s,t}} \approx \sum\limits_{r = 1}^R {\boldsymbol{a}_{p - i + \delta ,r}^{\left( 1 \right)}\boldsymbol{a}_{q - j + \delta ,r}^{\left( 2 \right)}\boldsymbol{a}_{s,r}^{\left( 3 \right)}\boldsymbol{a}_{t,r}^{\left( 4 \right)}}.
\end{aligned}
\end{equation} 

Substitute (\ref{Convolutional layer parameters}) into (\ref{output of convolutional layer}), the output can be sorted as
\begin{equation}
\vspace{-0.00in}
\resizebox{1\hsize}{!}{$\begin{aligned}
{\boldsymbol{y}_{i,j,t}} \approx \sum\limits_{r = 1}^R {\boldsymbol{a}_{t,r}^{\left( 4 \right)}\left( {\sum\limits_{p = i - \delta }^{i + \delta } {\boldsymbol{a}_{p - i + \delta ,r}^{\left( 1 \right)}\left( {\sum\limits_{q = j - \delta }^{j + \delta } {\boldsymbol{a}_{q - j + \delta ,r}^{\left( 2 \right)}\left( {\sum\limits_{s = 1}^{{I_S}} {\boldsymbol{a}_{s,r}^{\left( 3 \right)}{x_{p,q,s}}} } \right)} } \right)} } \right)}. \nonumber
\end{aligned}$}
\vspace{-0.00in}
\end{equation}

It can be seen that after CP decomposition, the weight of the convolution layer is actually equivalent to the series of four convolution layers. We use ${{\cal Y}^s}$, ${{\cal Y}^{sj}}$ and ${\mathcal{Y}^{sji}}$ to represent the intermediate results of the output of these sub-convolution layers, $\boldsymbol{y}_{p,q,r}^s$, $\boldsymbol{y}_{p,j,r}^{sj}$ and $\boldsymbol{y}_{i,j,r}^{sji}$ are respectively the corresponding elements, then we have
\begin{align}
\label{y_s}
\boldsymbol{y}_{p,q,r}^s &\approx \sum\limits_{s = 1}^{{I_S}} {\boldsymbol{a}_{s,r}^{\left( 3 \right)}} {\boldsymbol{x}_{p,q,s}}, \\
\label{y_sy}
\boldsymbol{y}_{p,j,r}^{sj} &\approx \sum\limits_{q = j - \delta }^{j + \delta } {\boldsymbol{a}_{q - j + \delta ,r}^{\left( 2 \right)}\boldsymbol{y}_{p,q,r}^s}, \\
\label{y_syx}
\boldsymbol{y}_{i,j,r}^{sji} &\approx \sum\limits_{p = i - \delta }^{i + \delta } {\boldsymbol{a}_{p - i + \delta ,r}^{\left( 1 \right)}} \boldsymbol{y}_{p,j,r}^{sj}, \\
\label{y_xyt}
\boldsymbol{y}_{i,j,t} &\approx \sum\limits_{r = 1}^R {\boldsymbol{a}_{t,r}^{\left( 4 \right)}\boldsymbol{y}_{i,j,r}^{sji}},
\end{align}
where the convolution represented by \eqref{y_s} and  \eqref{y_xyt} is a linear reorganization of the input feature graph, which is mapped to the output dimension. \eqref{y_sy} and \eqref{y_syx} respectively represent the convolution operation in the height and width directions.





Tensor decomposition enables some meaningful interpretation of the projection of the original high-dimensional complete model into the subspace, thereby extracting important information about the model.



\section{TDPFed Training Strategy}

According to \eqref{eq:clientloss}, the optimal  personalized model can be reformulated as
\begin{multline}
    {\hat {\boldsymbol{\theta}}} _k=
\operatorname{\arg \min }\limits_{{\boldsymbol{\theta} _k} \in {\mathbb{R}^{{I_1} \times {I_2} \times \dots \times {I_N}}}} \\
\left\{ {{f_k}\left( {{\boldsymbol{\theta} _k}} \right) 
+ \frac{\lambda }{2}{{\left\| {{\boldsymbol{\theta} _k} - \left[\kern-0.15em\left[ {{{\mathbf{A}}_k^{\left( 1 \right)}},\dots,{{\mathbf{A}}_k^{\left( N \right)}}} 
 \right]\kern-0.15em\right]} \right\|}^2}} \right\}.
\end{multline}

The client can then alternately optimizes the tensorized local model ${\mathbf{A}_k^{\left( n \right)}}$ by minimizing the local objective function ${F_k}$ under ${\hat{\boldsymbol{\theta}}_k}$, that is 
\begin{multline}
\label{Objective function Fk}
    {\left. {{F_k}\left( {{{\mathbf{A}}_k^{\left( 1 \right)}},\dots,{{\mathbf{A}}_k^{\left( N \right)}}} \right)} \right|_{{{\hat {\boldsymbol{\theta }}}_k}}} =
 {f_k}\left( {{{\hat {\boldsymbol{\theta}} }_k}} \right) + \\
\frac{\lambda }{2}{\left\| {{{\hat {\boldsymbol{\theta }}}_k} - \left[\kern-0.15em\left[ {{{\mathbf{A}}_k^{\left( 1 \right)}},\dots,{{\mathbf{A}}_k^{\left( N \right)}}} 
 \right]\kern-0.15em\right]} \right\|^2}.
\end{multline}

In this way, we only transmit tensorized local models to reduce the communication load, while preserving personalized models to improve performance on local non-IID data.

\subsection{Overall Training Progress}


On the client side, we first employ the tensorized local model $[\kern-0.15em [ {{{\mathbf{A}}_k^{( 1)}},\dots,{{\mathbf{A}}_k^{( N )}}} 
 ]\kern-0.15em]$ as the reference center point and minimize the objective function ${F_k}$ in \eqref{eq:Loss} to train the personalized model. Then, we train the tensorized local model under the condition of the optimal personalized model and then upload the factor matrices to the server. After the server aggregates the factor matrices, it broadcasts the aggregated tensor and factor matrices to clients as the personalized and tensorized local models, respectively. This training process is repeated until convergence.


\subsubsection{Stage 1: Global Initialization and Broadcasting}
First, the server initializes a global model $\{ {{\mathbf{A}}_0^{( 1 )}, \ldots ,{\mathbf{A}}_0^{( N )}} \}$  and broadcasts it to the clients. 


\subsubsection{Stage 2: At the $t$-th Global Communication Round}
Assume there are total $T$ global communication rounds. At each global communication round $t=1,\dots, T$, clients perform $\tau $ local update rounds to train the personalized models and the tensorized local models sequentially. In each local update round  $t'=1,\dots,\tau$, each client  $k=1,\dots,K$ first trains its personalized model ${\tilde{\boldsymbol{\theta}}_k}$ using $[\kern-0.15em[\tilde{\boldsymbol{\mathbf{A}}}_{k,t-1}^{( 1 )}, \dots, \tilde{\boldsymbol{\mathbf{A}}}_{k,t-1}^{( N )}]\kern-0.15em]$ as the reference center point.  Then, the client trains the tensorized local model $\{\tilde{\boldsymbol{\mathbf{A}}}_{k,t,t'}^{( n )}\}_{n=1}^{N}$ using the approximate solution ${\tilde{\boldsymbol{\theta}} _k}$ obtained. This stage is repeated $T$ rounds until convergence.

\subsubsection{Stage 3: Client Selection and Model Aggregation}
In federated learning, due to the uncertainty of connection, such as communication load or electric quantity, the server usually selects subset ${\mathcal{S}_t}$ of clients with the same size of $S$ for model aggregation and broadcasts the updated global model. Each selected client sends its updated local models $\{\tilde{\boldsymbol{\mathbf{A}}}_{k,t,\tau }^{( n )}\}_{n=1}^{N}$, $k\in \mathcal{S}_t$ to the server. At the server side, we design two aggregation strategies: aggregating factor matrix and aggregating composed tensor.  

\subsection{Personalized Model Training}

We randomly sample a mini-batch sample $\mathcal{B}_k$ in each local iteration, and make the gradient $\nabla\widetilde{f}_k(\boldsymbol{\theta}_k, \mathcal{B}_k)$ of its model prediction loss as an unbiased estimate of the ladder $\nabla{f}_k(\boldsymbol{\theta}_k)$ for the entire dataset. Define a new objective function over the mini-batch sample $\mathcal{B}_k$ as
\begin{multline}
\label{equ:Personalized-training}
{\tilde F_k}\left( {{\boldsymbol{\theta} _k};{\mathbf{A}}_{k,t,t'}^{\left( 1 \right)}, \ldots ,{\mathbf{A}}_{k,t,t'}^{\left( N \right)},{\mathcal{B}_k}} \right) \\
 \triangleq {\tilde f_k}\left( {{\boldsymbol{\theta} _k},{\mathcal{B}_k}} \right){\text{ + }}
\frac{\lambda }{2}{\left\| {{\boldsymbol{\theta} _k} - \left[\kern-0.15em\left[ {{\mathbf{A}}_{k,t,t'}^{\left( 1 \right)}, \ldots ,{\mathbf{A}}_{k,t,t'}^{\left( N \right)}} 
 \right]\kern-0.15em\right]} \right\|^2},    
\end{multline}
suppose we choose $\lambda$ such that ${\tilde F_k}$  is strongly convex with a condition number $\kappa$ (which quantifies how hard to optimize \eqref{equ:Personalized-training}, then we can apply gradient descent (resp. Nesterov's accelerated gradient descent) to obtain the personalized model ${\tilde{\boldsymbol{\theta}}_k}$ after $s$ number of computations, such that
\begin{equation}
\begin{aligned}
{\left\| {{{\tilde F}_k}\left( {{{\tilde{\boldsymbol{\theta}} }_k};{\mathbf{A}}_{k,t,t'}^{\left( 1 \right)}, \ldots ,{\mathbf{A}}_{k,t,t'}^{\left( N \right)},{\mathcal{B}_k}} \right)} \right\|^2} \leqslant \nu,
\label{solutionthea}
\end{aligned}
\end{equation} 
where $s: = \mathcal{O}\left( {\kappa \log \left( {{d \mathord{\left/
 {\vphantom {d v}} \right.
 \kern-\nulldelimiterspace} \nu}} \right)} \right)$ \cite{Bubeck15}, $d$ is the diameter of the search space, $\nu$ is the desired accuracy level. We replace the above optimal solution with an approximate solution ${\tilde{\boldsymbol{\theta}}_k}$.

\subsection{Tensorized Local Model Training}
We derive the gradient derivation of \eqref{eq:clientloss} with respect to parameter factor matrices $\{{{\mathbf{A}}_{k,t,t'}^{\left( n \right)}}\}_{n=1}^{N}$ to optimize the tensorized local models $[\kern-0.15em[ {{\mathbf{A}}_{k,t,t'}^{\left( 1 \right)}, \ldots ,{\mathbf{A}}_{k,t,t'}^{\left( N \right)}} ]\kern-0.15em]$. That is to say, we have the partial derivative ${{\partial {F_k}} \mathord{/
 {\vphantom {{\partial {F_k}} {\partial \boldsymbol{\mathbf{A}}_{k,t,t'}^{\left( n \right)}}}} 
 \kern-\nulldelimiterspace} {\partial \boldsymbol{\mathbf{A}}_{k,t,t'}^{\left( n \right)}}}$ of the local objective function with respect to $\boldsymbol{\mathbf{A}}_{k,t,r}^{( n )}$. Then the gradient descent based methods, including AdaGrad algorithm \cite{JohnCDuchi2010AdaptiveSM}, RMSProp algorithm \cite{tieleman2012lecture}, Adam algorithm \cite{kingma2014adam}, etc, can be applied. Inspired by tensor decomposition algorithms such as ALS, we alternately optimize each factor matrix $\boldsymbol{\mathbf{A}}_{k,t,t'}^{\left( n \right)}$. For convenience, we omit the subscript $\boldsymbol{\mathbf{A}}_{k,t,r}^{\left( n \right)}$, then the gradient ${{\partial {F_k}} \mathord{\left/
 {\vphantom {{\partial {F_k}} {\partial {\mathbf{A}^{\left( n \right)}}}}} \right.
 \kern-\nulldelimiterspace} {\partial {\boldsymbol{\mathbf{A}}^{\left( n \right)}}}}$ of client $k$'s factor matrix ${\boldsymbol{\mathbf{A}}^{\left( n \right)}}$ becomes
\begin{multline}
    \label{Training of local model-1}
\frac{\partial F_k}{\partial \textbf{A}^{(n)}} 
=\lambda\ \boldsymbol{\theta}_{k}^{(n)}
\Big(\textbf{A}^{(N)}\odot\dots\odot\textbf{A}^{(n+1)}\odot\textbf{A}^{(n-1)}\\
\odot\dots\odot\textbf{A}^{(1)} \Big)+
\textbf{A}^{(n)}\textbf{V}_n ,
\end{multline}
where
\begin{multline}
    \label{fanhua2}
\textbf{V}_n=
\textbf{A}^{(1)\mathsf{T}}\textbf{A}^{(1)}\ast\dots\ast
\textbf{A}^{(n-1)\mathsf{T}}\textbf{A}^{(n-1)}\ast\textbf{A}^{(n+1)\mathsf{T}}\textbf{A}^{(n+1)}\\
\ast\dots\ast
\textbf{A}^{(N)\mathsf{T}}\textbf{A}^{(N)}.
\end{multline}

The detailed derivation can be found in Appendix A.

Taking the gradient descent algorithm as an example, the updated factor matrix of $\boldsymbol{\mathbf{A}}_{k,t,t'}^{\left( n \right)}$ is given by
\begin{equation}
\begin{aligned}
\label{update-tdmodel}
\boldsymbol{\mathbf{A}}_{k,t,t'}^{\left( n \right)}{\text{ = }}\boldsymbol{\mathbf{A}}_{k,t,t'-1}^{\left( n \right)} - {\eta}\frac{{\partial {F_k}}}{{\partial \boldsymbol{\mathbf{A}}_{k,t,t'-1}^{\left( n \right)}}},
\end{aligned}
\end{equation}
where ${\eta}$ is the learning rate of the tensorized local model.
Since this gradient-based tensor training process requires several iterations, at the $t'$-th local update round , we train $s'$ number of computations to obtain the approximate 
$\{{{\tilde{\boldsymbol{\mathbf{A}}}}_{k,t,t'}^{( n )}}\}_{n=1}^{N}$.

\subsection{Global Model Aggregation}

Since it is difficult to guarantee that the factor matrices corresponding to each client can iterate in the same direction during gradient descent, directly calculating the average of each factor may incur a performance penalty. Therefore, we design the following two model aggregation strategies.


\subsubsection{Aggregating Factor Matrix (AFM)}

The server computes the average factor matrices ${\tilde{\mathbf{A}}}_{t{\text{ + }}1}^{\left( n \right)}$ of the global model alone each order $n$, $n = 1, \ldots ,N$, that is   
\begin{equation}
\begin{aligned}
\label{AFM-1}
{{\tilde{\mathbf{A}}}}_{t{\text{ + }}1}^{\left( n \right)} = \left( {1 - \beta } \right){\tilde{\mathbf{A}}}_t^{\left( n \right)} + \beta \sum\limits_{k \in {\mathcal{S}_t}} {\frac{{\left| {{\mathcal{B}_k}} \right|}}{{\sum\limits_{k \in {\mathcal{S}_t}} {\left| {{\mathcal{B}_k}} \right|} }}} {\tilde{\mathbf{A}}}_{k,t,\tau }^{\left( n \right)},
\end{aligned}
\end{equation} 
where $\beta$ is the aggregation coefficient, which controls the global model update ratio, including FedAvg's model averaging when $\beta=1$.

\subsubsection{Aggregating Composed Tensor (ACT)}

The server assembles the client's factor matrices into a full local model as
\begin{equation}
\begin{aligned}
\tilde{\boldsymbol{\omega}}_{t,\tau }^k = \left[\kern-0.15em\left[ {{\tilde{\mathbf{A}}}_{k,t,\tau }^{\left( 1 \right)}, \ldots ,{\tilde{\mathbf{A}}}_{k,t,\tau }^{\left( N \right)}} 
 \right]\kern-0.15em\right].
\end{aligned}
\end{equation} 

Then, the server calculates the average of the tensor $\tilde{\boldsymbol{\omega}}_{t,\tau }^k$, denoted as 
\begin{equation}
  \tilde{\boldsymbol{\omega}}_{t+1} = \left( {1 - \beta } \right)\left[\kern-0.15em\left[ {{\tilde{\mathbf{A}}}_t^{\left( 1 \right)}, \ldots ,{\tilde{\mathbf{A}}}_t^{\left( N \right)}} 
 \right]\kern-0.15em\right] + \beta \sum\limits_{k \in {\mathcal{S}_t}} {\frac{{\left| {{\mathcal{D}_k}} \right|}}{{\sum\limits_{k \in {\mathcal{S}_t}} {\left| {{\mathcal{D}_k}} \right|} }}} \tilde{\boldsymbol{\omega}}_{t,\tau }^k.  \nonumber
\end{equation}

Note that the aggregated global model can be denoted as $\tilde{\boldsymbol{\omega}}_{t+1}$ and ${\mathbf{A}}_{t + 1}^{\left( 1 \right)}, \ldots ,{\mathbf{A}}_{t + 1}^{\left( N \right)}$ for full and tensorized version, respectively, i.e.,
\begin{equation}
  \tilde{\boldsymbol{\omega}}_{t+1}
  \approx\left[\kern-0.15em\left[ {{\mathbf{A}}_{t + 1}^{\left( 1 \right)}, \ldots ,{\mathbf{A}}_{t + 1}^{\left( N \right)}} 
 \right]\kern-0.15em\right].  
\end{equation}

In summary, Algorithm 1 provides server and client training strategies, where $s$ is the number of iterations for the personalized model and $s'$ is the number of iterations for the tensorized local model.

\begin{algorithm}[h]
\renewcommand{\algorithmicrequire}{\textbf{Server executes:}}
\renewcommand{\algorithmicensure}
{\textbf{ClientUpdate ($k,\{ {\tilde {\boldsymbol{\mathbf{A}}}_{t}^{(n)}} \}_{n = 1}^N$):}}
\caption{TDPFed: Tensor Decomposition based Personalized Federated Learning Algorithm}
\label{alg:Framwork}
\begin{algorithmic}[1] 
\REQUIRE  

Input $\left\{ {{\mathcal{D}_k}} \right\}_{k = 1}^K$, $\lambda $, $\beta $, ${\eta}$, $\tau $, $T$, $s$, $s'$
\FOR{$t = 1,\dots,T$}
\FOR{$k = 1,\dots,K$ \textbf{in parallel}}
            \STATE $\{ {\tilde {\boldsymbol{\mathbf{A}}}_{k,t+1}^{(n)}} \}_{n = 1}^N\leftarrow$ \textbf{ClientUpdate ($k,\{ {\tilde {\boldsymbol{\mathbf{A}}}_{t}^{(n)}} \}_{n = 1}^N$)}
    
        \ENDFOR
        \STATE ${\mathcal{S}_t}\leftarrow$ (random set of $S$ Clients)
        
        \STATE Aggregating factor matrices to obtain $\{ {\tilde {\boldsymbol{\mathbf{A}}}_{t+1}^{(n)}} \}_{n = 1}^N$ according to (\ref{AFM-1})
        \ENDFOR

\ENSURE
\STATE $\{ {\tilde {\boldsymbol{\mathbf{A}}}_{k,t,0}^{(n)}} \}_{n = 1}^N=\{ {\tilde {\boldsymbol{\mathbf{A}}}_{t}^{(n)}} \}_{n = 1}^N$
\FOR{$t' = 1,\dots,\tau $}
            \STATE ${\mathcal{B}_k}\leftarrow$ (sample a mini-batch with size $|\mathcal{B}|$)

            \STATE Update ${\tilde {\boldsymbol{\theta}} _k}$ according to (\ref{equ:Personalized-training})

            \STATE Update $\{{{\tilde {\boldsymbol{\mathbf{A}}}}_{k,t,t'+1}^{( n )}}\}_{n=1}^{N}$ according to (\ref{update-tdmodel})
             
    \ENDFOR

    \STATE Return $\{ {\tilde {\boldsymbol{\mathbf{A}}}_{k,t,\tau }^{(n)}}\}_{n = 1}^N$ to the server
 \end{algorithmic}
\end{algorithm}

\section{Convergence Analysis}
We first show some useful assumptions, which are widely used in FL gradient calculation and convergence analysis~\cite{3-2020Personalized,2019SCAFFOLD,p2020Personalized}. Then, the convergence of TDPFed for nonconvex case is presented in Theorem~\ref{theorem_1}.
In addition, we show some intermediate results in the proof of the Theorem~\ref{theorem_1}.

\begin{assumption}\label{assumption_1} (Smoothness).
We assume that $f_k$ is nonconvex and L-smooth, (i.e., L-Lipschitz gradient), $\forall$ $\mathbf{A}_{k}^{\left( n \right)}$, $\mathbf{A^{\prime}}_{k}^{\left( n \right)}$:
\begin{equation*}
\left\|\nabla f_{k}(\mathbf{A}_{k}^{\left( n \right)})-\nabla f_{k}\left(\mathbf{A^{\prime}}_{k}^{\left( n \right)}\right)\right\| \leq L\left\|\mathbf{A}_{k}^{\left( n \right)}-\mathbf{A^{\prime}}_{k}^{\left( n \right)}\right\|,
\end{equation*}
where $L$ is called the smoothness parameter of $f_k$. 
\end{assumption}

\begin{assumption}\label{assumption_2} (Bounded variance).
The variance of stochastic gradients in each client is bounded
\begin{equation*}
\mathbb{E}_{\xi_{k}}\left[\left\|\nabla \tilde{f}_{k}\left(\mathbf{A}_{k}^{\left( n \right)} ; \xi_{k}\right)-\nabla f_{k}\left(\mathbf{A}_{k}^{\left( n \right)}\right)\right\|^{2}\right]\leq\gamma_{f}^{2}.
\end{equation*}
\end{assumption}

\begin{assumption}\label{assumption_3} (Bounded diversity).
The variance of local gradients to global gradient is bounded
\begin{equation*}
\sum_{k=1}^{K}\frac{|\mathcal{D}_k|}{|\mathcal{D}|}\left\|\nabla F_{k}(\mathbf{A}_{k}^{\left( n \right)})-\nabla F(\mathbf{A}^{\left( n \right)})\right\|^{2}\leq \sigma_{f}^{2}.
\end{equation*}
\end{assumption}
 
\begin{assumption}\label{assumption_4} (Bounded variance).
$\forall$ $k$, $\mathbf{A}_{k}^{\left( n \right)}$:
\begin{equation*}
\mathbb{E}\left[\left\|{g}_{k}(\mathbf{A}_{k,t,t'}^{\left( n \right)})-\nabla{F}_{k}(\mathbf{A}_{k,t,t'}^{\left( n \right)})\right\|^{2}\right]\leq\rho^2,
\end{equation*}
where \textcolor{black}{$g_k(\cdot)$ is the gradient of the client's local objective function with respect to the factors when the personalized model is ${\tilde{\boldsymbol{\theta}}_k}$.}
\end{assumption}


Assumption~\ref{assumption_1} is standard for convergence analysis, Assumptions~\ref{assumption_2} and \ref{assumption_3} are widely used in FL context in which $\gamma_{f}^{2}$ and $\sigma_{f}^{2}$ quantify the sampling noise and the diversity of client's data distribution, respectively~\cite{2019SCAFFOLD,p2020Personalized,2019Communication-1,DBLP-11}. Assumption~\ref{assumption_4} is similar with \cite{9762360}. Based on these assumptions, we have the following lemmas, which are proved in 
Appendix C-F.

\begin{lemma}\label{lemma_1}
Let  $\tilde{\theta}_{k}(\mathbf{A}_{k,t,t'}^{\left( n \right)})$ is a solution to \eqref{solutionthea}, we have
\begin{align*}
\mathbb{E}\left[\left\|\tilde{\theta}_{k}(\mathbf{A}_{k,t,t'}^{\left( n \right)})-\hat{\theta}_{k}(\mathbf{A}_{k,t,t'}^{\left( n \right)})\right\|^{2}\right]\leq\frac{2}{(\lambda-L)^{2}}\left(\frac{\gamma_{f}^{2}}{|\mathcal{B}|}+\nu\right).
\end{align*}
\end{lemma}

\begin{lemma}\label{lemma_2}
If Assumption \ref{assumption_1} holds, we have
\begin{align*}
\sum_{k=1}^{K}\frac{|\mathcal{D}_k|}{|\mathcal{D}|}\left\|\nabla F_{k}(\mathbf{A}_{k}^{\left( n \right)})-\nabla F(\mathbf{A}^{\left( n \right)})\right\|^{2}\leq2 \sigma_{f}^{2}+8\nu.
\end{align*}
\end{lemma}
This lemma provides the bounded diversity of $F_{k}$, it is related to $\sigma_{f}^{2}$ that needs to be bounded in Assumption \ref{assumption_3}. 

\begin{lemma}\label{lemma_3} Bounded diversity of $F_{k}$ w.r.t client sampling is given by
\begin{align*}
& \mathbb{E}_{\mathcal{S}_{t}}\left\|\frac{1}{S} \sum_{k \in \mathcal{S}^{t}} \nabla F_{k}\left(\mathbf{A}_{k,t}^{\left( n \right)}\right)-\nabla F\left(\mathbf{A}_{t}^{\left( n \right)}\right)\right\|^{2}\leq\\
& \sum_{k=1}^{K}\frac{|\mathcal{D}| / S-|\mathcal{D}_k|}{|\mathcal{D}|-|\mathcal{D}_k|}  \frac{|\mathcal{D}_k|}{|\mathcal{D}|}\left\|\nabla F_{k}\left(\mathbf{A}_{k,t}^{\left( n \right)}\right)-\nabla F\left(\mathbf{A}_{t}^{\left( n \right)}\right)\right\|^{2}. 
\end{align*}
\end{lemma}

\begin{lemma}\label{lemma_4} Bounded client drift error is
\begin{align*}
&\frac{1}{\tau}\sum_{k,t'}^{K,\tau}\frac{|\mathcal{D}_k|}{|\mathcal{D}|}\mathbb{E}\left[\left\|{g}_{k}(\mathbf{A}_{k,t,t'}^{\left( n \right)})-\nabla{F}_{k}(\mathbf{A}_{k,t}^{\left( n \right)})\right\|^{2}\right]\\
&\leq 2  \rho^{2}+\frac{16 \tilde{\eta}^{2} L_F^{2}}{\beta^{2}}\left(3\sum_{k=1}^{K} \frac{|\mathcal{D}_k|}{|\mathcal{D}|}\mathbb{E}\left[\left\|\nabla F_{k}\left(\mathbf{A}_{k,t}^{\left( n \right)}\right)\right\|^{2}\right]+\frac{2  \rho^{2}}{\tau}\right),
\end{align*}
where \textcolor{black}{$\tilde{\eta}\leq\frac{\beta}{2L_F}$ and $L_F$ is the smoothness parameter of $F_{k}$.}
\end{lemma}

Based on these lemmas, we have the following Theorem~\ref{theorem_1}, which provides the convergence of the proposed TDPFed.

\begin{theorem}\label{theorem_1} (Nonconvex and smooth TDPFed's convergence).
 Let Assumptions \ref{assumption_1}, \ref{assumption_2}, \ref{assumption_3}, \ref{assumption_4} hold. If $\tilde{\eta}\leq\frac{\beta}{2L_F}$, where $\beta \geq 1$, then we have
\begin{align*}
(a)&~\frac{1}{T}\sum_{t=0}^{T-1}\mathbb{E}\left[\left\|\nabla F\left(\mathbf{A}_{t}^{\left( n \right)}\right)\right\|^{2}\right]\\
&~~~~~~~~~~~~~~~\leq4\left(\frac{\Delta_F}{\tilde{\eta}T}+\frac{\tilde{\eta}^2}{\beta^{2}}C_1+\tilde{\eta}C_2+C_3\right),\\
(b)&~\frac{1}{T}\sum_{t=0}^{T-1}\sum_{k=1}^{K} \frac{|\mathcal{D}_k|}{|\mathcal{D}|} \mathbb{E}\left[\left\|\tilde{\theta}_{k,t}\left(\mathbf{A}_{k,t}^{\left( n \right)}\right)-\mathbf{A}_{t}^{\left( n \right)}\right\|^{2}\right]\\
&~~~~~~~~~~~~\quad\leq\frac{1}{T}\sum_{t=0}^{T-1}\mathbb{E}\left[\left\|\nabla F\left(\mathbf{A}_{t}^{\left( n \right)}\right)\right\|^{2}\right]+\mathcal{O}\left(D_1\right),
\end{align*}
where $\Delta_F\triangleq F(\mathbf{A}_{0}^{\left( n \right)})-\hat{F}$, \textcolor{black}{$\hat{F} = F(\hat{\mathbf{A}}^{\left( n \right)})$}, $C_1 = 32L_F^2\left(\frac{\rho^2}{\tau}+3 \sigma_{f}^{2}+12\nu\right)$, $C_2 = 3 L_{F} \frac{C}{J} \left(\sigma_{f}^{2}+4\nu\right)$, $C_3 = 2  \rho^{2}$, $D_1 = \frac{4}{(\lambda-L)^{2}}\left(\frac{\gamma_{f}^{2}}{|\mathcal{B}|}+\nu\right)+\frac{4}{\mu_F}(2\rho+2 \sigma_{f}^{2}+8\nu)$, $\mu_F$ is the strong convexity
parameter of $F_{k}$.
\begin{proof}
See Appendix G.
\end{proof}
\end{theorem}

\begin{remark}
Theorems~\ref{theorem_1}(a) shows the convergence result of the global-model. $\Delta_{F}$ denotes the initial error which can be reduced linearly. Theorem~\ref{theorem_1}(b) shows that the convergence of personalized client models in average to $\hat{\mathbf{A}}^{\left( n \right)}$ and radius $\mathcal{O}(D_1)$ for nonconvex, where $\hat{\mathbf{A}}^{\left( n \right)}$ \textcolor{black}{is seen as a ball of center of the personalized model}.  
\end{remark}

\section{Experiments}

\subsection{Experimental Setup}\label{sec:ES}

All the experiments are carried out within the framework of Pytorch on a computer with an Intel CPU i9-10920X and a single NVIDIA GeForce RTX3090 GPU. The dataset and model settings are as follows:

 \textbf{MNIST Dataset:}\label{sec:Mnist}
We use deep neural network (DNN) model training on the MNIST dataset. The MNIST dataset consists of 60,000 training samples and 10,000 test samples and contains ten classes of handwriting dataset. DNN consists of two fully connected layers, the size of which is $784 \times 100$ and $100 \times 10$. Relu and Softmax activation functions are used, respectively. In the experiment, the number of clients is 20, and the dataset is divided according to the non-IID way. Each client contains two types of samples. The number of clients participating in the aggregation in each round is 20, that is, all clients are selected. In addition to that, we use hyperparameters with $\left| {{\mathcal{B}}} \right| = 20$, $T = 800$, $\tau  = 23$, $\lambda  = 12$, $\beta  = 1.0$, $s = 5$, $s' = 17$, the learning rate of the tensorized local model is $\eta  = 0.0008$, the personalized learning rate of the personalized model is ${\eta _p} = 0.08$.

 \textbf{CIFAR-10 Dataset:} We use VGG8 model training on the CIFAR-10 dataset, consisting of 50000 training samples and 10000 test samples, including ten categories. The total number of clients is 10, and the dataset is divided in a non-IID way. Each client contains two types of samples. In the experiment, all clients are selected in each round of aggregation, we use hyperparameters with $\left| {{\mathcal{B}}} \right| = 20$, $T = 800$, $\tau  = 25$, $\lambda  = 14$, $\beta  = 1.0$, $s = 4$, $s' = 15$, the learning rate of the tensorized local model is $\eta  = 0.00004$, the personalized learning rate of the personalized model is ${\eta _p} = 0.03$.

Personalized models are trained with Nesterov's accelerated gradient descent algorithm, and tensorized local models are trained with Adam optimizer. The reason why the tensorized local model is trained using the Adam optimizer and the learning rate $\eta$ is different from the comparison algorithm in Section~\ref{sec:COM} is that in our designed framework, the local model is decomposed by iterating $s'$ times through the network training. With such a system design, the client can decompose the high-dimensional parameters of the local model into a low-dimensional space, which can reduce the traffic in the uplink. Due to the different feature dimensions for training, the value of the learning rate of the tensorized local model used is different from other algorithms.  

\subsection{Compression Rate Definition}

For the network structure decomposed by CP, the compression rate is defined as the ratio of the original network's total parameters to the network's full parameters after tensor decomposition.


\textbf{Fully Connected Layer:} $\boldsymbol{W} \in {{\mathbb R}^{{I_{{\rm{out}}}} \times {I_{\rm{in}}}}}$. If the number of weight parameters is ${I_{{\text{out}}}} \times {I_{{\text{in}}}}$ and the number of decomposed parameters is ${I_{{\text{out}}}} \times R + {I_{{\text{in}}}} \times R$, then the compression rate ($CR$) of the fully connected layer is given by
\begin{equation}
\begin{aligned}
\label{eq29}
CR=\frac{{{I_{{\text{out}}}}\times{I_{{\text{in}}}}}}{{R\times( {{I_{{\text{out}}}} + {I_{{\text{in}}}}})}}.
\end{aligned}
\end{equation}

\textbf{Convolutional Layer:} ${\cal W} \in {{\mathbb R}^{{I_d} \times {I_d} \times {I_S} \times {I_T}}}$. If the number of weight parameters is ${I_d} \times {I_d} \times {I_S} \times {I_T}$ and the number of decomposed parameters is ${I_d} \times R + {I_d} \times R + {I_S} \times R + {I_T} \times R$, then the compression rate of the convolutional layer is 

\begin{equation}
\begin{aligned}
\label{eq30}
CR=\frac{{{I_d}\times{I_d}\times{I_S}\times{I_T}}}{{R\times( {{I_d} + {I_d} + {I_S} + {I_T}})}}.
\end{aligned}
\end{equation}

Considering compression rates 1.5 and 2, denoted as $(\times1.5)$ and $(\times2)$, respectively. And the relative CP ranks of network weight at each layer are represented by R$1$ and R$2$, respectively.  According to Section \ref{sec:TLM}, we obtain two sets of tensorized local models. 
The parameter dimension and CP rank are shown in Tab. \ref{tab:DNN_VGG}.

\begin{table}[!ht]
\centering
\caption{CP Decomposition Rank of DNN and VGG8 Networks}
\label{tab:DNN_VGG}
\begin{tabular}{cccccc}
\toprule[1.5pt]
Model & Layer & Dimension & R1($\times$1.5) & R2($\times$2) \\
\midrule[0.75pt]
\multirow{2}{*}{DNN}
&fc 1  & 784$\times$100 & 59 & 44 \\
& fc 2  & 100$\times$10 & 6 & 5 \\
\midrule[0.75pt]
\multirow{8}{*}{VGG8}
&conv 1  & 3$\times$32$\times$3$\times$3 & 14 & 11\\
& conv 2  & 32$\times$64$\times$3$\times$3 & 120 & 90\\
&conv 3  & 64$\times$128$\times$3$\times$3 & 248 & 186\\
& conv 4  & 128$\times$256$\times$3$\times$3 & 504 & 378\\
&conv 5  & 256$\times$256$\times$3$\times$3 & 759 & 569\\
& fc 1  & 256$\times$256 & 85 & 64\\
&fc 2  & 256$\times$256 & 85 & 64\\
& fc 3  & 256$\times$10 & 6 & 5\\
\bottomrule[1.5pt]
\end{tabular}
\end{table}




\subsection{Model Aggregation Strategy Comparison}\label{sec:MASC}

To compare the performance of the two aggregation strategies, AFM and ACT, we conduct experiments on the MNIST and CIFAR-10 datasets using the same hyperparameters as Section~\ref{sec:ES}. Tab.~\ref{AFM-ACT} shows the test accuracy of the proposed two aggregation strategies, AFM and ACT. On the MNIST dataset, the accuracies of AFM and ACT are very similar. On the CIFAR-10 dataset, the AFM aggregation strategy shows apparent advantages.








 


 

\begin{table}[]
\centering
\caption{Test accuracy of the proposed two aggregation strategies, AFM and ACT.}
\label{AFM-ACT}
\begin{tabular}{ccccccccc}
\toprule[1.5pt]
Algorithm

&  MNIST      &  CIFAR-10

\\

\midrule[0.75pt]

\rowcolor{gray!10} AFM$(\times2)$             & 98.97{\%}                    & 91.01{\%}                   \\

 ACT$(\times2)$                         & 99.07{\%}        & 74.80{\%}                     \\
 
 \midrule[0.75pt]

\rowcolor{gray!10} AFM$(\times1.5)$             & 99.08{\%}                    & 91.16{\%}                   \\

 ACT$(\times1.5)$                         & 98.07{\%}        & 77.31{\%}                     \\
 
\bottomrule[1.5pt]
\end{tabular}
\end{table}


Overall, we can see that the advantage of the aggregation strategy of AFM is significantly higher than that of ACT, which averages the factor matrix before doing aggregation. We found that even if there is a slight difference between the averaged factor matrix and the original factor matrix, there will be a significant error in the aggregated high-dimensional tensor. This exciting finding encourages us to propose the AFM aggregation method for complex classification tasks. Therefore, in the following experiments, we adopt the aggregation method of AFM.

\subsection{Performance Comparison with Other Federal Frameworks}\label{sec:COM}

We compare the proposed TDPFed with the FedAvg \cite{1-Communication}, and pFedMe \cite{3-2020Personalized} algorithms on the MNIST and CIFAR-10 datasets. Specifically, we set $\eta=$ 3e-4 for TDPFed on MNIST, and $\eta=$ 4e-5 for TDPFed on CIFAR-10. The learning rate for FedAvg is 0.05 on MNIST and 0.01 on CIFAR-10. The learning rate of the local model for pFedMe is 0.05 on MNIST and 0.01 on CIFAR-10. For other hyperparameters, we use the same settings as Section~\ref{sec:ES}. We tune the hyperparameters to achieve the highest test accuracy for each algorithm, and each set of experiments is an average of three experimental results. Fig. \ref{MNIST_Com(AFM)_Accuracy}, Fig. \ref{CIFAR10_Com(AFM)_Accuracy} and Tab. \ref{ALL-RESULTS} present the test accuracy under different compression rates, while Fig. \ref{MNIST_Com(AFM)_Loss} and Fig. \ref{CIFAR10_Com(AFM)_Loss} show the training loss in different settings.

To ensure the same compression rate, we let FedAvg randomly select some parameters for transmission according to the same proportion in the process of uploading and downloading model parameters, which is recorded as (s.a.r.) in the result. Although FedAvg implements random upload, it may not reduce the uplink traffic of the entire federated learning system, because the position where the parameter amount is 0 still needs to be encoded and transmitted, so the number of bits used may not be reduced. However, our proposed learning architecture radically compresses the transmitted traffic.

According to these results, TDPFed outperforms FedAvg and pFedMe in model accuracy in most settings and has a faster convergence rate. It shows that the proposed TDPFed can reduce the communication cost while achieving good accuracy on non-IID data.
Although the results of pFedMe are better than ours on the CIFAR-10 dataset in Fig.~\ref{CIFAR10_Com(AFM)_Loss}, our algorithm reduces the communication load.

\begin{figure}
        \center
        \scriptsize
        \begin{tabular}{cc}
                \includegraphics[width=4cm]{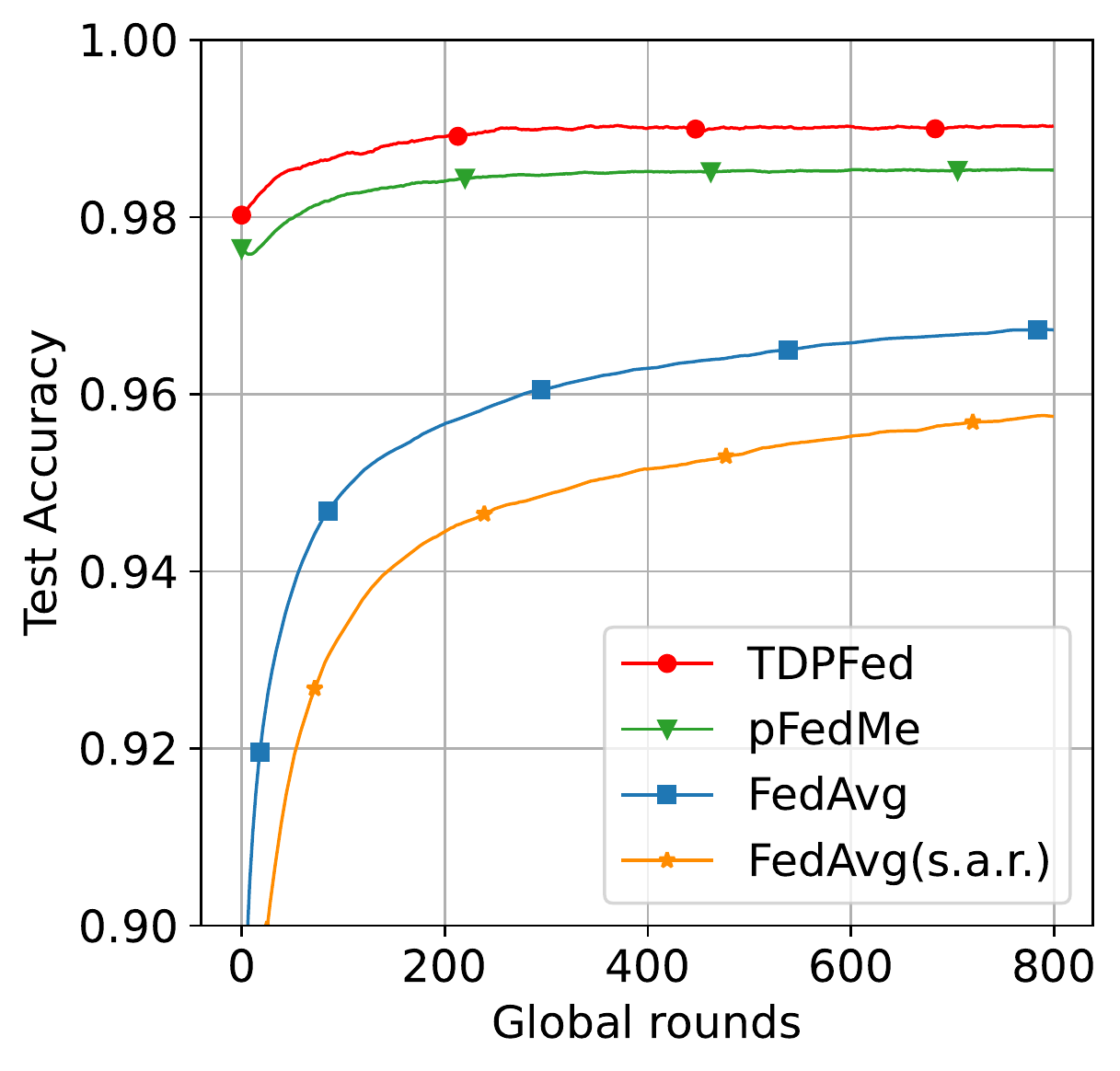} &    \includegraphics[width=4cm]{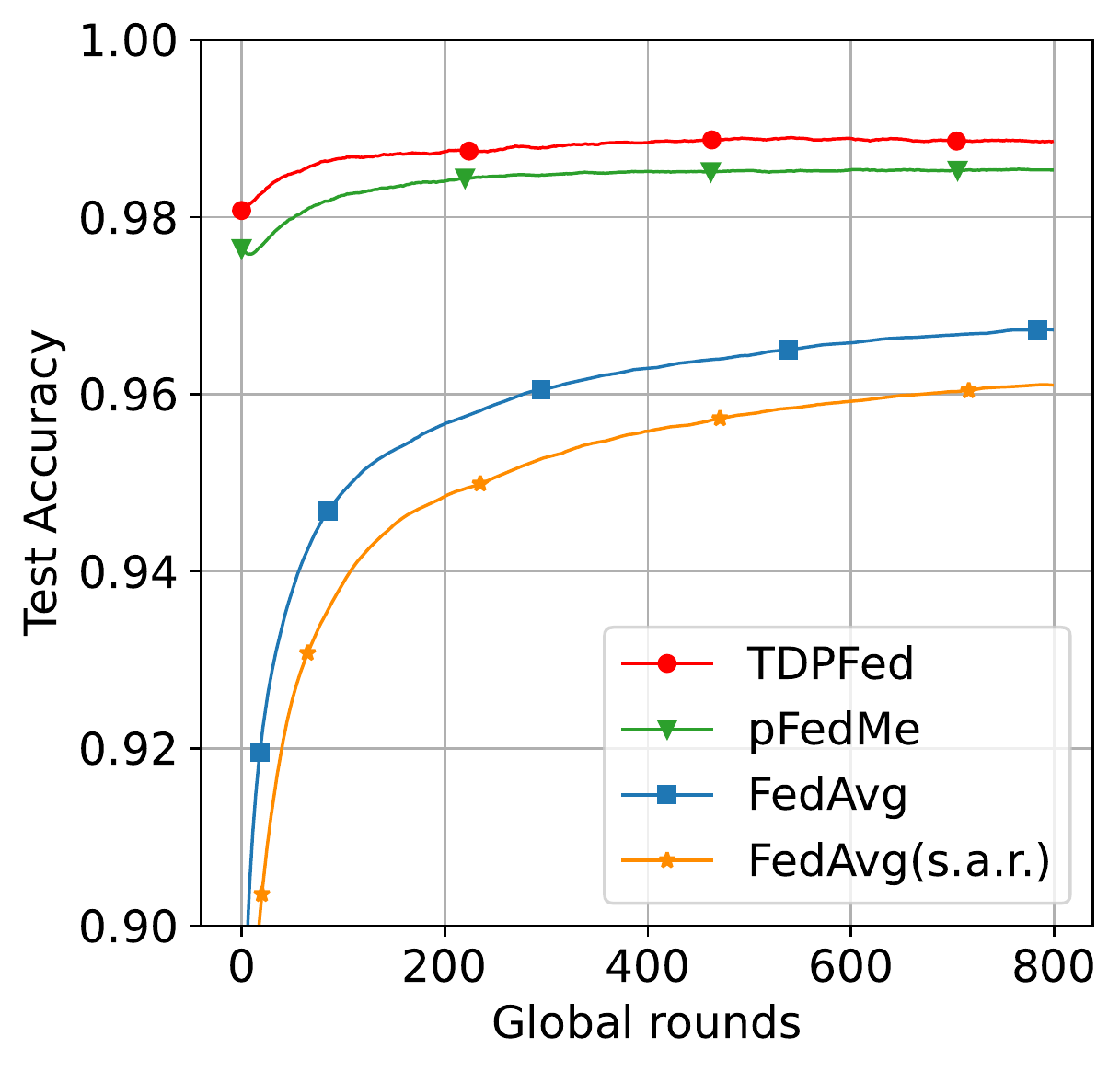}        \\
                (a) $(\times2)$ & (b) $(\times1.5)$ \\
        \end{tabular}
        \caption{Comparison of the test accuracy of the proposed TDPFed, FedAvg, and pFedMe algorithm on the MNIST dataset.}
        \label{MNIST_Com(AFM)_Accuracy}
        \vspace{-0.5em}
\end{figure}

\begin{figure}
        \center
        \scriptsize
        \begin{tabular}{cc}
                \includegraphics[width=4cm]{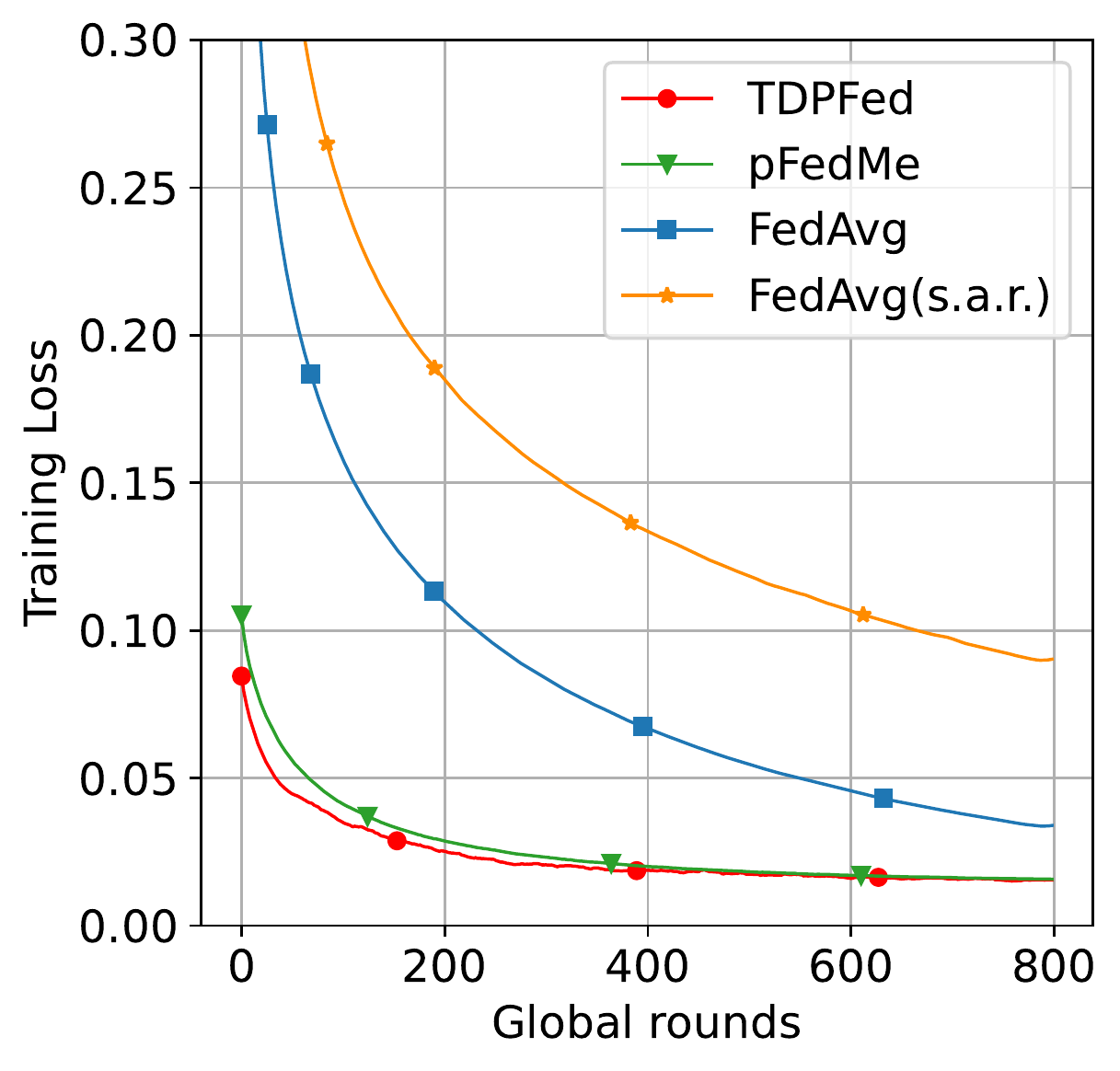} &    \includegraphics[width=4cm]{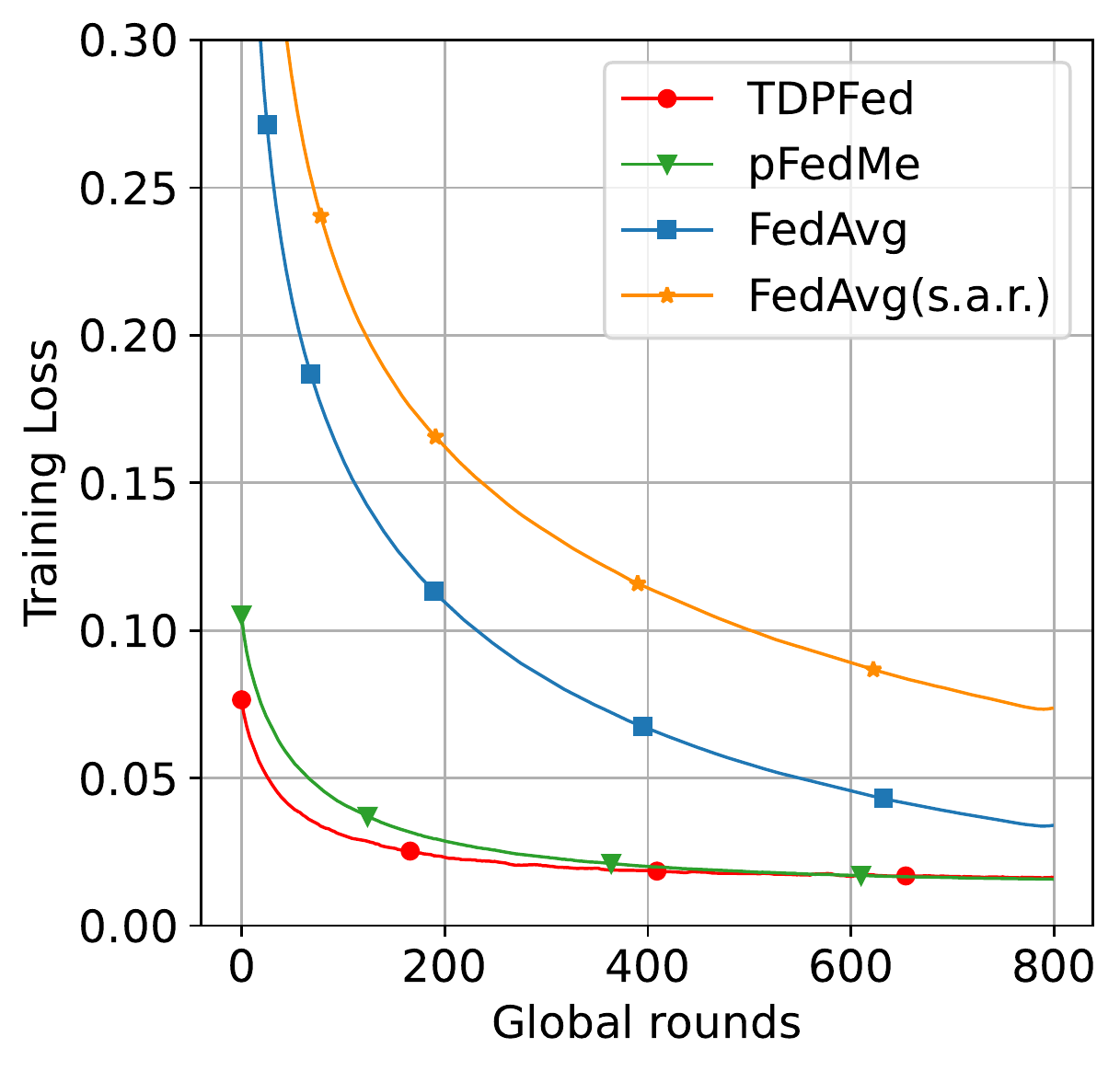}        \\
                (a) $(\times2)$ & (b) $(\times1.5)$ \\
        \end{tabular}
        \caption{Comparison of training loss of the proposed TDPFed, FedAvg, and pFedMe algorithm on the MNIST dataset.}
        \label{MNIST_Com(AFM)_Loss}
        \vspace{-0.5em}
\end{figure}

\begin{figure}
        \center
        \scriptsize
        \begin{tabular}{cc}
                \includegraphics[width=4cm]{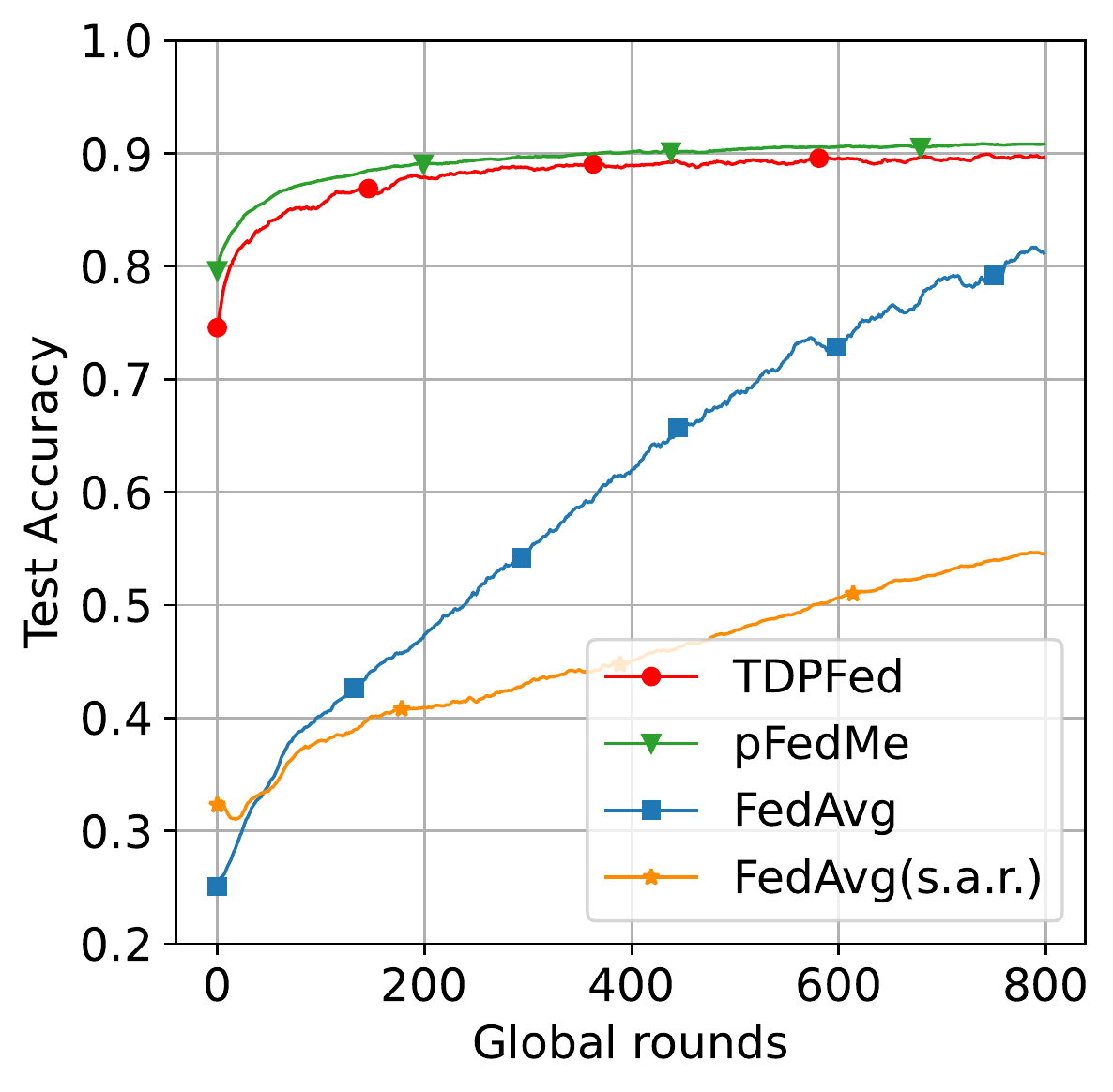} &    \includegraphics[width=4cm]{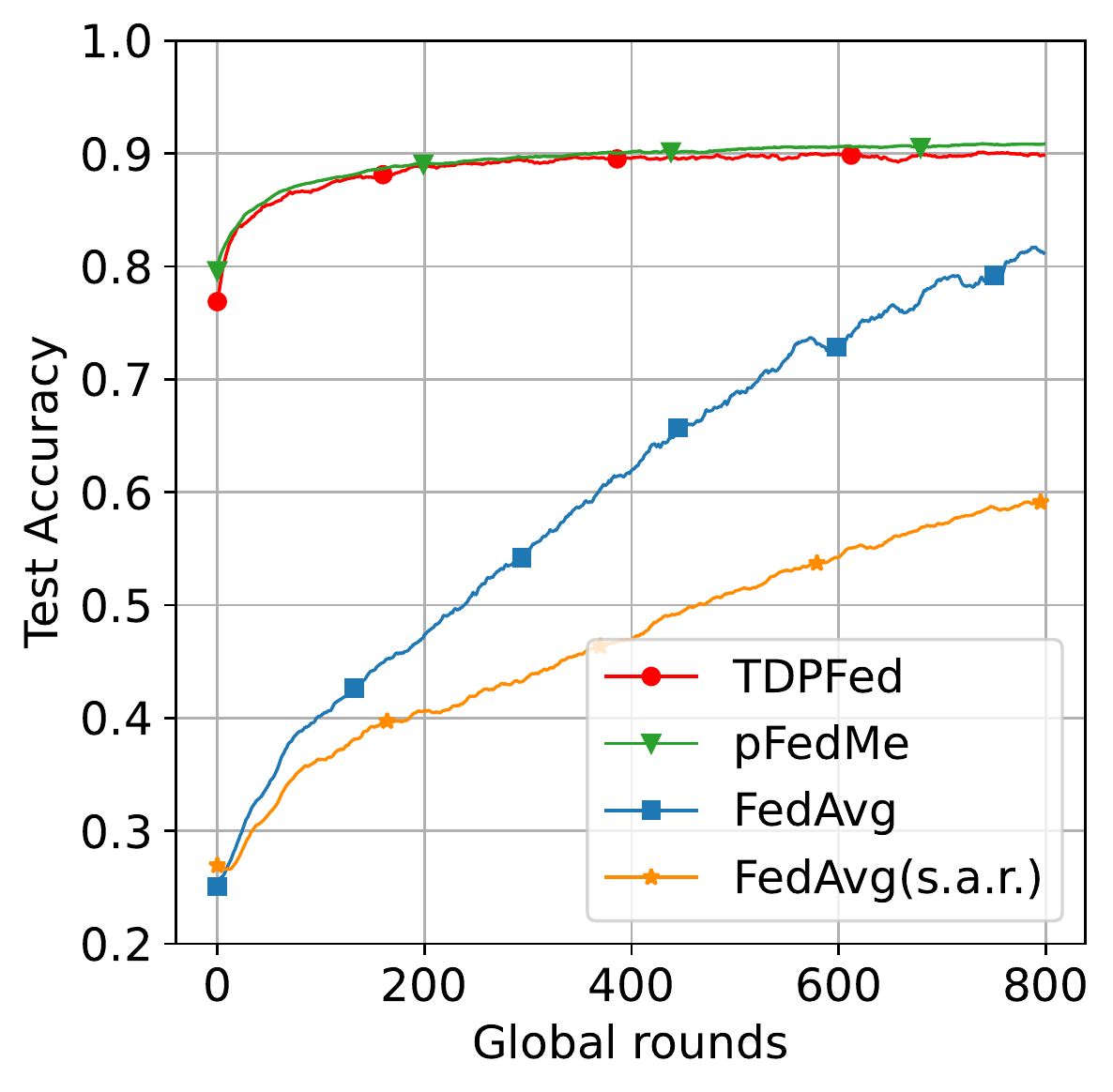}        \\
                (a) $(\times2)$ & (b) $(\times1.5)$ \\
        \end{tabular}
        \caption{Comparison of the test accuracy of the proposed TDPFed, FedAvg, and pFedMe algorithm on the CIFAR-10 dataset.}
        \label{CIFAR10_Com(AFM)_Accuracy}
        \vspace{-0.5em}
\end{figure}

\begin{figure}
        \center
        \scriptsize
        \begin{tabular}{cc}
                \includegraphics[width=4cm]{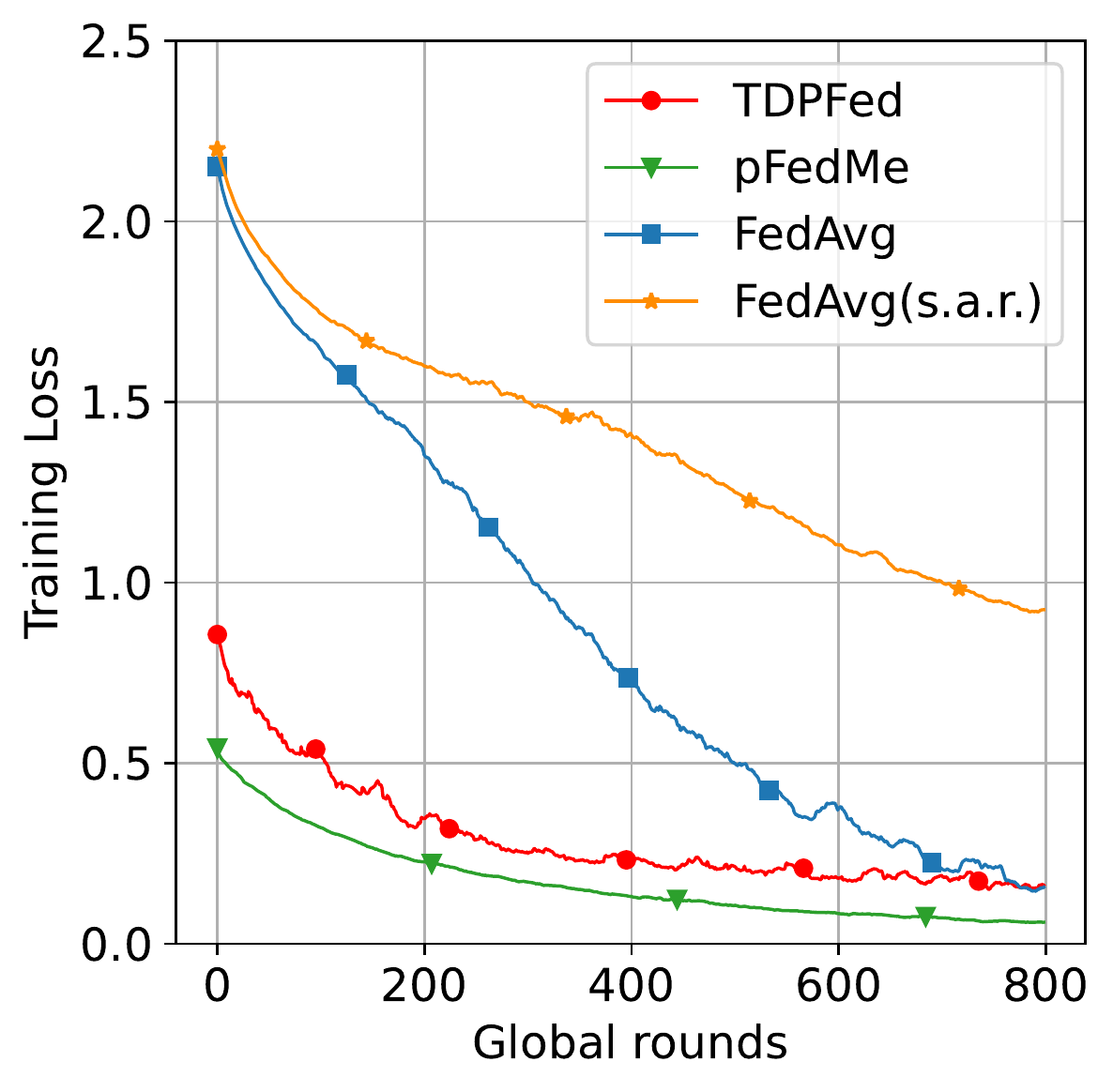} &    \includegraphics[width=4cm]{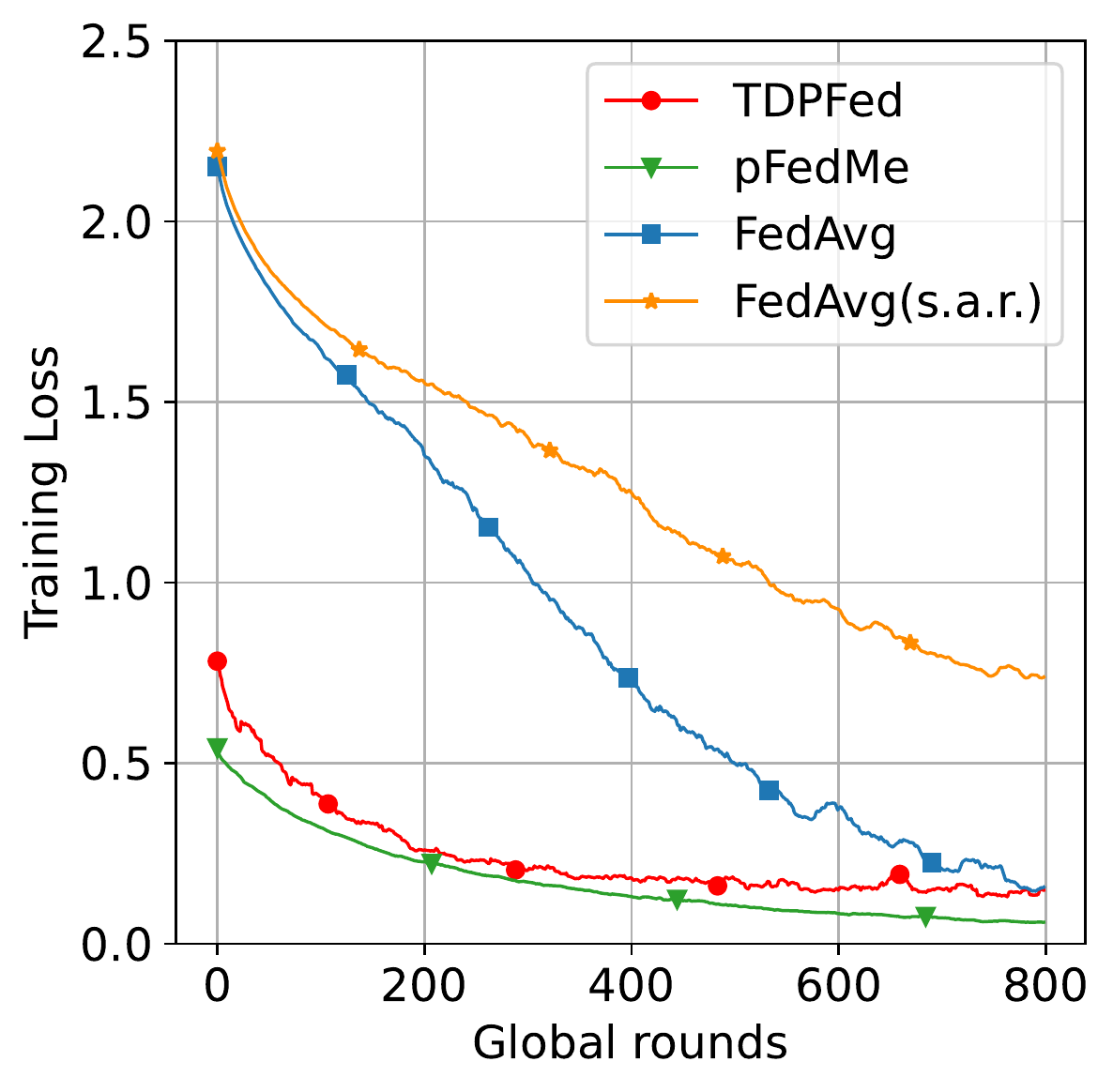}        \\
                (a) $(\times2)$ & (b) $(\times1.5)$ \\
        \end{tabular}
        \caption{Comparison of training loss of the proposed TDPFed, FedAvg, and pFedMe algorithm on the CIFAR-10 dataset.}
        \label{CIFAR10_Com(AFM)_Loss}
        \vspace{-0.5em}
\end{figure}








                                                  


\begin{table}[!ht]
\small
\centering
\caption{Test accuracy of TDPFed, FedAvg, and pFedMe algorithms.}
\label{ALL-RESULTS}
\begin{tabular}{ccc}
\toprule[1.5pt]
Algorithm &  MNIST      &  CIFAR-10

\\

\midrule[0.75pt]

FedAvg               & 
96.78{\%}                    & 83.43{\%} \\

pFedMe                         & 98.64{\%}        & 91.41{\%}\\

\multirow{2}{*}{FedAvg(s.a.r)}
&96.15{\%}$(\times1.5)$  & 61.00{\%}$(\times1.5)$  \\
& 95.80{\%}$(\times2)$  & 55.51{\%}$(\times2)$  \\

\multirow{2}{*}{TDPFed}
&99.04{\%}$(\times1.5)$ & 91.16{\%}$(\times1.5)$  \\
& 99.16{\%}$(\times2)$ & 91.01{\%}$(\times2)$  \\
\bottomrule[1.5pt]
\end{tabular}
\end{table}

\subsection{Algorithm Performance with Different Parameters}

\begin{table}[!ht]
\centering
\caption{Hyperparameter Study}
\label{tab:Haperparameters}
\begin{tabular}{cccc}
\toprule[1.5pt]
\multicolumn{2}{c}{MNIST} & \multicolumn{2}{c}{CIFAR-10}\\
\cmidrule(r){1-2} \cmidrule(r){3-4} 

settings &  Accuracy  

& settings &  Accuracy
\\

\midrule[0.75pt]
\rowcolor{gray!10} baseline & 98.97{\%} & baseline & 91.01{\%}\\
$\lambda=8$  & 98.77{\%} & $\lambda=10$ & 90.16{\%}\\
\rowcolor{gray!10}$\lambda=10$  & 98.91{\%} & $\lambda=18$ & 90.44{\%}\\
$\beta=1.4$  & 98.95{\%} & $\beta=1.4$ & 90.77{\%}\\
\rowcolor{gray!10}$\beta=1.8$  & 99.11{\%} & $\beta=1.8$ & 91.36{\%}\\
$s=2$  & 97.79{\%} & $s=2$ & 90.58{\%}\\
\rowcolor{gray!10}$s=8$  & 97.85{\%} & $s=6$ & 89.64{\%}\\
$s'=12$  & 98.94{\%} & $s'=10$ & 90.20{\%}\\
\rowcolor{gray!10}$s'=22$  & 98.88{\%} & $s'=20$ & 90.65{\%}\\
\bottomrule[1.5pt]
\end{tabular}
\end{table}

We compare the effects of different hyperparameters such as $\lambda$, $\beta$, $s$, and $s'$ on the convergence of TDPFed. When fine-tuning one hyperparameter, we fixed other hyperparameters using the setting in Section~\ref{sec:ES}, the results are shown in Tab.~\ref{tab:Haperparameters}.



\textbf{Regularization parameter $\lambda $:} If $\lambda$ is too large or too small, the model's performance will degrade. This is because $\lambda$ in \eqref{Objective function Fk} controls the distance between the tensorized local model and the personalized model, and also determines the gradient of the tensorized local model. Therefore, it is necessary to select an appropriate $\lambda$ according to different scenarios.

\textbf{Aggregation coefficient $\beta $:} When $\beta=1$, the model aggregation is consistent with the FedAvg. And increasing $\beta $ appropriately can improve the convergence speed of the local model.

\textbf{Iterations of the personalized model $s$:} When $s$ increases from 2 to 4, the model performance also increases, but when $s$ continues to grow from 4 to 6, the model performance drops significantly. The reason is that the training process of the personalized model is affected by the $\lambda $ in (\ref{equ:Personalized-training}). When the tensorized local model is unchanged, the Euclidean distance between the personalized model and the tensorized local model cannot be too far.


\textbf{Iterations of the tensorized local model $s'$:}  The training process of the local model is that the tensorized local model parameters are iterated $s'$ times along the direction of the personalized model parameters. Then the average global model among clients is obtained through the aggregation of the server. It can be seen that with the increase of $s'$, the model's accuracy is slightly decreased because the increase of $s'$ means the increase in communication cost, and the appropriate value of $s'$ need to be selected to balance calculation and communication.


\section{Conclusion}
In this paper, we propose TDPFed as a communication-efficient personalized FL framework that performs robust on non-IID data, improving communication efficiency.
Our approach uses a newly designed tensorized local model with low dimension factor matrices and a bi-level objective function to help decouple personalized optimization from the global model learning. Clients optimize the personalized and tensorized local models in parallel to adapt to the statistical diversity issue and only send tensorized local models to the server for FL's aggregation.
Moreover, a distributed learning strategy with two different model aggregation approaches, AFM and ACT, is well designed for the proposed TDPFed framework with many simulation experiments and discussions.
Experimental results demonstrate that TDPFed can achieve excellent accuracy while reducing communication costs. The theoretical analysis shows that TDPFed’s convergence rate is state-of-the-art with linear speedup.

\appendix

\subsection{Gradient Derivation of Model Factor}
We copy \eqref{Training of local model-1} here as
\begin{align}
\label{eq_27}
\frac{\partial F_k}{\partial \textbf{A}^{(n)}} = &~
\lambda\frac
{\partial}{\partial\textbf{A}^{(n)}}
\frac{1}{2}||\boldsymbol{\theta}_k-
\llbracket \textbf{A}^{(1)} \dots \textbf{A}^{(N)}  \rrbracket||^2 \nonumber \\
=&~\lambda\{-\boldsymbol{\theta}_{k}^{(n)}
(\textbf{A}^{(N)}\odot\dots\odot\textbf{A}^{(n+1)}\odot\textbf{A}^{(n-1)}\odot\dots \nonumber \\
&~~~~~~~~~~~~~\odot\textbf{A}^{(1)})+
\textbf{A}^{(n)}\textbf{V}_n\},
\end{align}
where
\begin{multline}
    \label{eq:definition of V}
\textbf{V}_n=
\textbf{A}^{(1)\mathsf{T}}\textbf{A}^{(1)}\ast\dots\ast
\textbf{A}^{(n-1)\mathsf{T}}\textbf{A}^{(n-1)}\ast\textbf{A}^{(n+1)\mathsf{T}}\textbf{A}^{(n+1)}\\
\ast\dots\ast
\textbf{A}^{(N)\mathsf{T}}\textbf{A}^{(N)}.
\end{multline}

The detailed derivation progress is as follows. Let $f=\frac{1}{2}||\boldsymbol{\theta}_k-
\llbracket \textbf{A}^{(1)} \dots \textbf{A}^{(N)}  \rrbracket||^2$, then it can be denoted as
\begin{equation}
f=
\frac{1}{2}\underbrace{||\boldsymbol{\theta}_k||^2}_{f_1}-
\underbrace{\left<\boldsymbol{\theta}_k,\llbracket\textbf{A}^{(1)} \dots \textbf{A}^{(N)}\rrbracket\right>}_{f_2}+
\frac{1}{2}\underbrace{||\llbracket\textbf{A}^{(1)} \dots \textbf{A}^{(N)}\rrbracket||^2}_{f_3}. \nonumber
\end{equation}

Note that $\frac{\partial f_1}{\partial {\boldsymbol{a}}_r^{(n)}}=\textbf{0}$. $f_2$ is the inner product between $\boldsymbol{\theta}_k$ and the CP approximation. The partial derivation of $f_2$ to $\boldsymbol{a}_r^{(n)}$ is
\begin{align}
\frac{\partial f_2}{\partial {\boldsymbol{a}}_r^{(n)}}
=&~\frac{\partial}{\partial {\boldsymbol{a}}_r^{(n)}}\left<\boldsymbol{\theta}_k, \sum_{r=1}^{R}{\boldsymbol{a}}_r^{(1)}\circ\dots\circ {\boldsymbol{a}}_r^{(N)}\right> \nonumber \\
=&~\frac{\partial}{\partial {\boldsymbol{a}}_r^{(n)}}
\sum_{r=1}^{R}\sum_{i_1=1}^{I_1}\sum_{i_2=1}^{I_2}\dots\sum_{i_N=1}^{I_N}
\boldsymbol{\theta}_{k,i_{1}i_{2}\dots i_{N}} \nonumber \\
&~{\boldsymbol{a}}_{i_{1},r}^{(1)}{\boldsymbol{a}}_{i_{2},r}^{(2)}\dots {\boldsymbol{a}}_{i_{N},r}^{(N)} \nonumber \\
=&~\frac{\partial}{\partial {\boldsymbol{a}}_r^{(n)}}\sum_{r=1}^{R}\boldsymbol{\theta}_k\times_{1} {\boldsymbol{a}}_r^{(1)}\times\dots\times_{N}{\boldsymbol{a}}_r^{(N)} \nonumber \\
=&~\frac{\partial}{\partial {\boldsymbol{a}}_r^{(n)}}\sum_{r=1}^{R}
(\boldsymbol{\theta}_k\times_{1} {\boldsymbol{a}}_r^{(1)}\times\dots\times_{n-1}{\boldsymbol{a}}_r^{(n-1)} \nonumber \\
&~\times_{n+1}{\boldsymbol{a}}_r^{(n+1)}\times\dots\times_{N}{\boldsymbol{a}}_r^{(N)})^{\mathsf{T}}
\boldsymbol{a}_r^{(n)} \nonumber \\
=&~\boldsymbol{\theta}_k\times_{1} {\boldsymbol{a}}_r^{(1)}\times\dots\times_{n-1}{\boldsymbol{a}}_r^{(n-1)}\times_{n+1}{\boldsymbol{a}}_r^{(n+1)} \nonumber \\
&~\times\dots\times_{N}{\boldsymbol{a}}_r^{(N)},
\end{align}
where $\frac{\partial f_2}{\partial {\boldsymbol{a}}_r^{(n)}}$ results in a vector which $\in \mathbb{R}^{I_n}$ and is equivalent to $\boldsymbol{\theta}_{k}^{(n)}{\boldsymbol{a}}_r^{(1)}\otimes\dots\otimes {\boldsymbol{a}}_r^{(n-1)}\otimes {\boldsymbol{a}}_r^{(n+1)}\otimes\dots\otimes {\boldsymbol{a}}_r^{(N)}$.

The partial derivation of $f_3$ to ${\boldsymbol{a}}_r^{(n)}$ is given by
\begin{align}
\label{proof:partial derivative of f3}
\frac{\partial f_3}{\partial {\boldsymbol{a}}_r^{(n)}}
=&~\frac{\partial}{\partial {\boldsymbol{a}}_r^{(n)}}\left<\sum_{r=1}^{R}{\boldsymbol{a}}_r^{(1)}\circ\dots\circ {\boldsymbol{a}}_r^{(N)},
\sum_{r=1}^{R}{\boldsymbol{a}}_r^{(1)}\circ\dots\circ {\boldsymbol{a}}_r^{(N)}\right>  \nonumber \\
=&~\frac{\partial}{\partial {\boldsymbol{a}}_r^{(n)}}
\sum_{k=1}^{R}\sum_{l=1}^{R}\prod_{m=1}^{N}{\boldsymbol{a}}_k^{(m)\mathsf{T}}{\boldsymbol{a}}_l^{(m)}  \nonumber \\
=&~\frac{\partial}{\partial {\boldsymbol{a}}_r^{(n)}}
(\prod_{m=1}^{N}{\boldsymbol{a}}_r^{(m)\mathsf{T}}{\boldsymbol{a}}_r^{(m)}
+2\sum_{\substack{l=1\\l\neq r}}^{R}\prod_{m=1}^{N}{\boldsymbol{a}}_r^{(m)\mathsf{T}}{\boldsymbol{a}}_l^{(m)}  \nonumber \\
&~+\sum_{\substack{k=1\\k\neq r}}^{R}\sum_{\substack{l=1\\l\neq r}}^{R}
\prod_{m=1}^{N}{\boldsymbol{a}}_k^{(m)\mathsf{T}}{\boldsymbol{a}}_l^{(m)})  \nonumber \\
=&~2(\prod_{\substack{m=1\\m\neq n}}^{N}{\boldsymbol{a}}_r^{(m)\mathsf{T}}{\boldsymbol{a}}_r^{(m)}){\boldsymbol{a}}_r^{(n)}
+2\sum_{\substack{l=1\\l\neq r}}^{R}
(\prod_{\substack{m=1\\m\neq n}}^{N}{\boldsymbol{a}}_r^{(m)\mathsf{T}}{\boldsymbol{a}}_l^{(m)}){\boldsymbol{a}}_l^{(n)}  \nonumber \\
=&~2\sum_{l=1}^{R}(\prod_{\substack{m=1\\m\neq n}}^{N}{\boldsymbol{a}}_r^{(m)\mathsf{T}}{\boldsymbol{a}}_l^{(m)}){\boldsymbol{a}}_l^{(n)}.
\end{align}

Actually, $\textbf{V}_n$ in (\ref{eq:definition of V}) is an $R\times R$ matrix whose $(r,l)$ entry is $\prod_{\substack{m=1\\m\neq n}}{\boldsymbol{a}}_r^{(m)\mathsf{T}}{\boldsymbol{a}}_l^{(m)}$, then (\ref{proof:partial derivative of f3}) can be written as the sum of each product between the elements in the $r$-th row of $\textbf{V}_n$ and the according ${\boldsymbol{a}}_l^{(n)}$, which is the $l$-th column of $\textbf{A}^{(n)}$ as
\begin{equation}
\frac{\partial f_3}{\partial {\boldsymbol{a}}_r^{(n)}}
=2\sum_{l=1}^{R}\textbf{V}_{n,rl}{\boldsymbol{a}}_l^{(n)}.
\end{equation}

Then we have
\begin{equation}
\begin{aligned}
\frac{\partial f}{\partial {\boldsymbol{a}}_r^{(n)}}
=&-\boldsymbol{\theta}_{k}^{(n)}{\boldsymbol{a}}_r^{(1)}\otimes\dots\otimes {\boldsymbol{a}}_r^{(n-1)}\otimes {\boldsymbol{a}}_r^{(n+1)}
\otimes\dots\\
&~~~~~~~~~~~~~~~~~~~~~~\otimes {\boldsymbol{a}}_r^{(N)}
+\sum_{l=1}^{R}\textbf{V}_{n,rl}{\boldsymbol{a}}_l^{(n)}.
\end{aligned}
\end{equation}

Since that ${\boldsymbol{a}}_r^{(n)}$ constitutes each column of $\textbf{A}^{(n)}$, the partial derivative of $f$ with respect to the factor matrix $\textbf{A}^{(n)}$ can be associated as
\begin{equation}
\begin{aligned}
\frac{\partial f}{\partial \textbf{A}^{(n)}} =
&-\boldsymbol{\theta}_{k}^{(n)}(\textbf{A}^{(N)}\odot\dots\odot\textbf{A}^{(n+1)}\odot\textbf{A}^{(n-1)}\odot\dots\\
&~~~~~~~~~~~~~~~~~~~~~~~~~~~\odot\textbf{A}^{(1)})+\textbf{A}^{(n)}\textbf{V}_{n}.
\label{eq_34}
\end{aligned}
\end{equation}

Substitute \eqref{eq_34} into \eqref{eq_27} to finish the derivation. 


\subsection{Some Useful Results}
\begin{proposition}\label{Proposition1}
The objective function of client $k$ is ${F_k(\cdot)}$, which is not only $\frac{\lambda(2\lambda+L)}{\lambda+L}$-strongly convex but also $\frac{\lambda(2\lambda-L)}{\lambda-L}$-smooth, with the condition that $\lambda\textgreater L$.
\end{proposition}
\begin{proof}
The objective function of client $k$ is
\begin{align*}
&~{F_k}\left( {{{\mathbf{A}}_k^{\left( 1 \right)}}, \ldots ,{{\mathbf{A}}_k^{\left( N \right)}}} \right) \\
\triangleq&~ \mathop {\min }\limits_{{\boldsymbol{\theta} _k}} \left\{ {{f_k}\left( {{\boldsymbol{\theta} _k}} \right) + \frac{\lambda }{2}{{\left\| {{\boldsymbol{\theta} _k} - \left[\kern-0.15em\left[ {{\mathbf{A}_k^{\left( 1 \right)}},\dots,{{\mathbf{A}}_k^{\left( N \right)}}} 
 \right]\kern-0.15em\right]} \right\|}^2}} \right\}\\
=&~\mathop {\min }\limits_{{\boldsymbol{\theta} _k}} \left\{ {{f_k}\left( {{\boldsymbol{\theta} _k}} \right)+\frac{\lambda}{2}||\boldsymbol{\theta}_k||^2-\lambda\left<\boldsymbol{\theta}_k,\left[\kern-0.15em\left[ {{\mathbf{A}_k^{\left( 1 \right)}},\dots,{{\mathbf{A}}_k^{\left( N \right)}}} 
 \right]\kern-0.15em\right]\right>} \right\}\\
&~+\frac{\lambda}{2}{{\left\| {\left[\kern-0.15em\left[ {{\mathbf{A}_k^{\left( 1 \right)}},\dots,{{\mathbf{A}}_k^{\left( N \right)}}} 
 \right]\kern-0.15em\right]} \right\|}^2}\\
=&~{-\lambda}\mathop {\min }\limits_{{\boldsymbol{\theta} _k}} \Big\{{\left<\boldsymbol{\theta}_k,\left[\kern-0.15em\left[ {{\mathbf{A}_k^{\left( 1 \right)}},\dots,{{\mathbf{A}}_k^{\left( N \right)}}} 
 \right]\kern-0.15em\right]\right>} \\
 &~-(\frac{1}{2}||\boldsymbol{\theta}_k||^2 + \frac{1}{\lambda}{f_k}\left( {{\boldsymbol{\theta} _k}} \right))\Big\} +\frac{\lambda}{2}{{\left\| {\left[\kern-0.15em\left[ {{\mathbf{A}_k^{\left( 1 \right)}},\dots,{{\mathbf{A}}_k^{\left( N \right)}}} 
 \right]\kern-0.15em\right]} \right\|}^2}.
\end{align*}

On the one hand, let $\Psi_{k}( {{\boldsymbol{\theta} _k}} )=\frac{1}{2}||\boldsymbol{\theta}_k||^2 + \frac{1}{\lambda}{f_k}( {{\boldsymbol{\theta} _k}} )$, and due to Assumption \ref{assumption_1}, $\Psi_{k}( {{\boldsymbol{\theta} _k}} )$ is $(1+\frac{L}{\lambda})$-smooth, its conjugate function ${\Psi^*_{k}({{\mathbf{A}}_k^{(n)}})}=\mathop {\max }\limits_{{\boldsymbol{\theta} _k}}\{{<\boldsymbol{\theta}_k,[\kern-0.15em[ {{\mathbf{A}_k^{( 1 )}},\dots,{{\mathbf{A}}_k^{( N )}}} 
 ]\kern-0.15em]>}-(\frac{1}{2}||\boldsymbol{\theta}_k||^2 + \frac{1}{\lambda}{f_k}( {{\boldsymbol{\theta} _k}} ))\}$ is $\frac{\lambda}{\lambda+L}$-strongly convex, $n=1,\dots,N$. Therefore, $F_k({{\mathbf{A}}_k^{(n)}})$ is $\frac{\lambda(2\lambda+L)}{\lambda+L}$-strongly convex, $n=1,\dots,N$.

On the other hand, let $\Psi_{k}\left( {{\boldsymbol{\theta} _k}} \right)=\frac{1}{2}||\boldsymbol{\theta}_k||^2 + \frac{1}{\lambda}{f_k}\left( {{\boldsymbol{\theta} _k}} \right)$, and due to Assumption \ref{assumption_1}, $\Psi_{k}\left( {{\boldsymbol{\theta} _k}} \right)$ is $(1-\frac{L}{\lambda})$-strongly convex, its conjugate function ${\Psi^*_{k}({{\mathbf{A}}_k^{(n)}})}=\mathop {\max }\limits_{{\boldsymbol{\theta} _k}}\{{<\boldsymbol{\theta}_k,[\kern-0.15em[ {{\mathbf{A}_k^{( 1 )}},\dots,{{\mathbf{A}}_k^{( N )}}} 
 ]\kern-0.15em]>}-(\frac{1}{2}||\boldsymbol{\theta}_k||^2 + \frac{1}{\lambda}{f_k}\left( {{\boldsymbol{\theta}_k}} \right))\}$ is $\frac{\lambda}{\lambda-L}$-smooth, $n=1,\dots,N$. Therefore, $F_k({{\mathbf{A}}_k^{(n)}})$ is $\frac{\lambda(2\lambda-L)}{\lambda-L}$-smooth, with the condition that $\lambda > L$, $n=1,\dots,N$.
\end{proof}

\begin{proposition}\label{Proposition2}
If a function ${F_k(\cdot)}$ is $L_F$-smooth and $\mu_F$-strongly convex, $\forall$ $\mathbf{A}_{k}^{\left( n \right)}$, $\mathbf{A^{\prime}}_{k}^{\left( n \right)}$, we have the following useful inequalities \cite{3-2020Personalized}, in respective order,
\begin{align*}
&\left\|\nabla F_{k}(\mathbf{A}_{k}^{\left( n \right)})-\nabla F_{k}\left(\mathbf{A^{\prime}}_{k}^{\left( n \right)}\right)\right\|^2\\
\leq & ~ 2L_F\left(F_{k}(\mathbf{A}_{k}^{\left( n \right)})- F_{k}\left(\mathbf{A^{\prime}}_{k}^{\left( n \right)}\right)\right.\\
&~~~~\quad\left.-\left<\nabla F_{k}\left(\mathbf{A^{\prime}}_{k}^{\left( n \right)}\right), \mathbf{A}_{k}^{\left( n \right)}-\mathbf{A^{\prime}}_{k}^{\left( n \right)}\right>\right),
\end{align*}
and
\begin{align*}
\mu_F\left\|\mathbf{A}_{k}^{\left( n \right)}-\mathbf{A^{\prime}}_{k}^{\left( n \right)}\right\|\leq
\left\|\nabla F_{k}(\mathbf{A}_{k}^{\left( n \right)})-\nabla F_{k}\left(\mathbf{A^{\prime}}_{k}^{\left( n \right)}\right)\right\|.
\end{align*}
\end{proposition}

\begin{proposition}\label{Proposition3}
For any vector $x_i\in{\mathbb{R}}^{d}, i =1,\dots,M$, by Jensen's inequality, we have
\begin{align*}
\left\|\sum_{i=1}^{M}x_i \right\|^2 \leq M \sum_{i=1}^{M}\left\|x_i\right\|^2.
\end{align*}
\end{proposition}

\subsection{Proof of Lemma~\ref{lemma_1}}
\begin{proof}
Define $$h_{k}\left(\theta_{k} ; \mathbf{A}_{k,t,t'}^{\left( n \right)}\right):=f_{k}\left(\theta_{k}\right)+\frac{\lambda}{2}{\left\| {{\boldsymbol{\theta} _k} - \left[\kern-0.15em\left[ {{\mathbf{A}_{k,t,t'}^{\left( 1 \right)}},\dots,\mathbf{A}_{k,t,t'}^{\left( N \right)}} 
 \right]\kern-0.15em\right]} \right\|}^{2}.$$

Then $h_{k}\left(\theta_{k} ; \mathbf{A}_{k,t,t'}^{\left( n \right)}\right)$ is $(\lambda-L)$-strongly convex with its unique solution $\hat{\theta}_{k}(\mathbf{A}_{k,t,t'}^{\left( n \right)})$, $n=1, \dots, N$. Then, by Proposition \ref{Proposition2}, we have
\begin{align*}
&\left\|\tilde{\theta}_{k}(\mathbf{A}_{k,t,t'}^{\left( n \right)})-\hat{\theta}_{k}(\mathbf{A}_{k,t,t'}^{\left( n \right)})\right\|^{2} \\
\leq &~ \frac{1}{(\lambda-L)^{2}}\left\|\nabla h_{k}\left(\tilde{\theta}_{k} ; \mathbf{A}_{k,t,t'}^{\left( n \right)}\right)\right\|^{2}\\
\leq &~\frac{2}{(\lambda-L)^{2}}\left(\left\|\nabla h_{k}\left(\tilde{\theta}_{k} ; \mathbf{A}_{k,t,t'}^{\left( n \right)}\right)-\nabla \tilde{h}_{k}\left(\tilde{\theta}_{k} ; \mathbf{A}_{k,t,t'}^{\left( n \right)}, \mathcal{B}_{k}\right)\right\|^{2}\right.\\
&~\left.\quad+\left\|\nabla \tilde{h}_{k}\left(\tilde{\theta}_{k} ; \mathbf{A}_{k,t,t'}^{\left( n \right)}, \mathcal{B}_{k}\right)\right\|^{2}\right)\\
\leq &~ \frac{2}{(\lambda-L)^{2}}\left(\left\|\nabla \tilde{f}_{k}\left(\tilde{\theta}_{k} ; \mathcal{B}_{k}\right)-\nabla f_{k}\left(\tilde{\theta}_{k}\right)\right\|^{2}+\nu\right)\\
= &~ \frac{2}{(\lambda-L)^{2}}\left(\frac{1}{|\mathcal{B}|^{2}}\left\|\sum_{\xi_{k} \in \mathcal{B}_{k}} \nabla \tilde{f}_{k}\left(\tilde{\theta}_{k} ; \xi_{k}\right)-\nabla f_{k}\left(\tilde{\theta}_{k}\right)\right\|^{2}+\nu\right),
\end{align*}
where the second inequality is by Proposition \ref{Proposition3}. Taking expectation to both sides, we have
\begin{align*}
&\mathbb{E}\left[\left\|\tilde{\theta}_{k}(\mathbf{A}_{k,t,t'}^{\left( n \right)})-\hat{\theta}_{k}(\mathbf{A}_{k,t,t'}^{\left( n \right)})\right\|^{2}\right]\\ =&~\frac{2}{(\lambda-L)^{2}}\Bigg(\frac{1}{|\mathcal{B}|^{2}} \sum_{\xi_{k} \in \mathcal{B}_{k}} \mathbb{E}_{\xi_{k}}\left[\left\|\nabla \tilde{f}_{k}\left(\tilde{\theta}_{k} ; \xi_{k}\right)-\nabla f_{k}\left(\tilde{\theta}_{k}\right)\right\|^{2}\right] \\
&~+\nu \Bigg) 
\leq \frac{2}{(\lambda-L)^{2}}\left(\frac{\gamma_{f}^{2}}{|\mathcal{B}|}+\nu\right),
\end{align*}
where the first equality is due to $\mathbb{E}[\|\sum_{i=1}^{M} X_{i}-\mathbb{E}\left[X_{i}\right]\|^{2}]=\sum_{i=1}^{M} \mathbb{E}\left[\left\|X_{i}-\mathbb{E}\left[X_{i}\right]\right\|\right]^{2}$ with $M$ independent random variables $X_{i}$ and the unbiased estimate $\mathbb{E}\left[\nabla \tilde{f}_{k}\left(\tilde{\theta}_{k} ; \xi_{k}\right)\right]=\nabla f_{k}(\tilde{\theta}_{k})$, and the last inequality is due to Assumption \ref{assumption_2}.
\end{proof}

\subsection{Proof of Lemma~\ref{lemma_2}}
\begin{proof}
Let $\textbf{H}_n=(\textbf{A}^{(N)}\odot\dots\odot\textbf{A}^{(n+1)}\odot\textbf{A}^{(n-1)}\odot\dots\odot\textbf{A}^{(1)})$, then $\frac{\partial F_k}{\partial \textbf{A}^{(n)}} =\lambda\{-(\boldsymbol{\theta}_{k})_{(n)}\textbf{H}_n+\textbf{A}^{(n)}\textbf{V}_n\}$, then we have
\begin{align*}
&\left\|\nabla F_{k}(\mathbf{A}_{k}^{\left( n \right)})-\nabla F(\mathbf{A}^{\left( n \right)})\right\|^{2} \\
=&~ \Bigg \|\lambda\{-(\boldsymbol{\theta}_{k})_{(n)}\textbf{H}_{n,k}+\textbf{A}_k^{(n)}\textbf{V}_{n,k}\}\\
&~~~~~\quad- \sum_{j=1}^{K}\frac{|\mathcal{D}_j|}{|\mathcal{D}|} \lambda\{-(\boldsymbol{\theta}_{j})_{(n)}\textbf{H}_{n,j}+\textbf{A}_j^{(n)}\textbf{V}_{n,j}\} \Bigg\|^{2} \\
= &~ \left\|\nabla f_{k}\left(\hat{\theta}_{k}(\mathbf{A}_{k}^{\left( n \right)})\right)-\sum_{j=1}^{K}\frac{|\mathcal{D}_j|}{|\mathcal{D}|} \nabla f_{j}\left(\hat{\theta}_{j}(\mathbf{A}_{j}^{\left( n \right)})\right)\right\|^{2} \\
=&~ 2\left\|\nabla f_{k}\left(\hat{\theta}_{k}(\mathbf{A}_{k}^{\left( n \right)})\right)-\sum_{j=1}^{K}\frac{|\mathcal{D}_j|}{|\mathcal{D}|} \nabla f_{j}\left(\hat{\theta}_{k}(\mathbf{A}_{k}^{\left( n \right)})\right)\right\|^{2}\\
&~~~\quad+2\left\|\sum_{j=1}^{K}\frac{|\mathcal{D}_j|}{|\mathcal{D}|} \nabla f_{j}\left(\hat{\theta}_{k}(\mathbf{A}_{k}^{\left( n \right)})\right)-\nabla f_{j}\left(\hat{\theta}_{j}(\mathbf{A}_{j}^{\left( n \right)})\right)\right\|^{2},
\end{align*}
where the second equality is due to the first-order condition $\nabla f_{k}\left(\hat{\theta}_{k}(\mathbf{A}_{k}^{\left( n \right)})\right)-\lambda\{-(\boldsymbol{\theta}_{k})_{(n)}\textbf{H}_{n,k}+\textbf{A}_k^{(n)}\textbf{V}_{n,k}\}=0$, and the last equality is due to Proposition \ref{Proposition3}. Taking the average over the number of clients, we have
\begin{align*}
&~~~\sum_{k=1}^{K}\frac{|\mathcal{D}_k|}{|\mathcal{D}|}\left\|\nabla F_{k}(\mathbf{A}_{k}^{\left( n \right)})-\nabla F(\mathbf{A}^{\left( n \right)})\right\|^{2}\\
&\leq 2 \sigma_{f}^{2}
+ \sum_{k=1}^{K} \sum_{j=1}^{K}\frac{2|\mathcal{D}_k||\mathcal{D}_j|}{|\mathcal{D}|^2}\left\|\nabla f_{j}\left(\hat{\theta}_{k}(\mathbf{A}_{k}^{\left( n \right)})\right)\right.\\
&\quad\left.-\nabla f_{j}\left(\hat{\theta}_{j}(\mathbf{A}_{j}^{\left( n \right)})\right)\right\|^{2}\\
&\leq 2 \sigma_{f}^{2}
+ \sum_{k=1}^{K} \sum_{j=1}^{K}\frac{2|\mathcal{D}_k||\mathcal{D}_j|}{|\mathcal{D}|^2}2\left(\left\|\nabla f_{j}\left(\hat{\theta}_{k}(\mathbf{A}_{k}^{\left( n \right)})\right)\right\|^{2}\right.\\&\left.\quad+\left\|\nabla f_{j}\left(\hat{\theta}_{j}(\mathbf{A}_{j}^{\left( n \right)})\right)\right\|^{2}\right)\\
&\leq 2 \sigma_{f}^{2}
+ \sum_{k=1}^{K} \sum_{j=1}^{K}\frac{4|\mathcal{D}_k||\mathcal{D}_j|}{|\mathcal{D}|^2}(2\nu)\\
&= 2 \sigma_{f}^{2}+8\nu,
\end{align*}
where the third inequality is due to \eqref{solutionthea}.
\end{proof}

\subsection{Proof of Lemma~\ref{lemma_3}}
\begin{proof}
We use similar proof arguments in \cite{WZ20}'s Lemma 5 as follows
\begin{align*}
&~~~\mathbb{E}_{\mathcal{S}_{t}}\left\|\frac{1}{S} \sum_{k \in \mathcal{S}^{t}} \nabla F_{k}\left(\mathbf{A}_{k,t}^{\left( n \right)}\right)-\nabla F\left(\mathbf{A}_{t}^{\left( n \right)}\right)\right\|^{2}\\
&=\frac{1}{S^{2}} \mathbb{E}_{\mathcal{S}_{t}}\left\|\sum_{k=1}^{K} \mathbb{I}_{k \in S_{t}}\left(\nabla F_{k}\left(\mathbf{A}_{k,t}^{\left( n \right)}\right)-\nabla F\left(\mathbf{A}_{t}^{\left( n \right)}\right)\right)\right\|^{2} \\
&=\frac{1}{S^{2}}\left[ \sum_{k=1}^{K} \mathbb{E}_{\mathcal{S}_{t}}\left[\mathbb{I}_{k \in S_{t}}\right]\left\|\nabla F_{k}\left(\mathbf{A}_{k,t}^{\left( n \right)}\right)-\nabla F\left(\mathbf{A}_{t}^{\left( n \right)}\right)\right\|^{2}\right. \\
&\left.\quad+\sum_{k \neq j} \mathbb{E}_{\mathcal{S}_{t}}\left[\mathbb{I}_{k \in S_{t}} \mathbb{I}_{j \in S_{t}}\right]\left\langle\nabla F_{k}\left(\mathbf{A}_{k,t}^{\left( n \right)}\right)-\nabla F\left(\mathbf{A}_{t}^{\left( n \right)}\right),\right.\right.\\
&\left.\left.\quad\nabla F_{j}\left(\mathbf{A}_{j,t}^{\left( n \right)}\right)-\nabla F\left(\mathbf{A}_{t}^{\left( n \right)}\right)\right\rangle\right] \\
&=\frac{1}{S}\sum_{k=1}^{K}\frac{|\mathcal{D}_k|}{|\mathcal{D}|}\left\|\nabla F_{k}\left(\mathbf{A}_{k,t}^{\left( n \right)}\right)-\nabla F\left(\mathbf{A}_{t}^{\left( n \right)}\right)\right\|^{2}\\
&\quad+\sum_{k \neq j} \frac{(S-1)|\mathcal{D}_k|^2}{S|\mathcal{D}| (|\mathcal{D}|-|\mathcal{D}_k|)}\left\langle\nabla F_{k}\left(\mathbf{A}_{k,t}^{\left( n \right)}\right)-\nabla F\left(\mathbf{A}_{t}^{\left( n \right)}\right), \right.\\
&\left.\quad\nabla F_{j}\left(\mathbf{A}_{j,t}^{\left( n \right)}\right)-\nabla F\left(\mathbf{A}_{t}^{\left( n \right)}\right)\right\rangle \\
&=\frac{1}{S}\sum_{k=1}^{K}\left(1-\frac{(S-1)|\mathcal{D}_k|}{|\mathcal{D}|-|\mathcal{D}_k|}\right) \left\|\nabla F_{k}\left(\mathbf{A}_{k,t}^{\left( n \right)}\right)-\nabla F\left(\mathbf{A}_{t}^{\left( n \right)}\right)\right\|^{2} \\
&=\sum_{k=1}^{K}\frac{|\mathcal{D}| / S-|\mathcal{D}_k|}{|\mathcal{D}|-|\mathcal{D}_k|}  \frac{|\mathcal{D}_k|}{|\mathcal{D}|}\left\|\nabla F_{k}\left(\mathbf{A}_{k,t}^{\left( n \right)}\right)-\nabla F\left(\mathbf{A}_{t}^{\left( n \right)}\right)\right\|^{2},
\end{align*}
where the third equality is due to $\mathbb{E}_{\mathcal{S}_{t}}\left[\mathbb{I}_{k \in S_{t}}\right]=\mathbb{P}\left(k \in S_{t}\right)=\frac{S|\mathcal{D}_k|}{|\mathcal{D}|}$ and $\mathbb{E}_{\mathcal{S}_{t}}\left[\mathbb{I}_{k \in S_{t}} \mathbb{I}_{j \in S_{t}}\right]=$ $\mathbb{P}\left(k, j \in S_{t}\right)=\frac{S(S-1)|\mathcal{D}_k|^2}{|\mathcal{D}|(|\mathcal{D}|-|\mathcal{D}_k|)}$ for all $k \neq j$, and the fourth equality is by 
\begin{align*}
    & \sum_{k=1}^{K}\left\|\nabla F_{k}\left(\mathbf{A}_{k,t}^{\left( n \right)}\right)-\nabla F\left(\mathbf{A}_{t}^{\left( n \right)}\right)\right\|^{2} \nonumber \\
&~~~ +\sum_{k \neq j}\Big\langle\nabla F_{k}\left(\mathbf{A}_{k,t}^{\left( n \right)}\right)-\nabla F\left(\mathbf{A}_{t}^{\left( n \right)}\right), \nonumber \\
&~~~~~~~~~~~~~~~~~~~~~~~~~ \nabla F_{j}\left(\mathbf{A}_{j,t}^{\left( n \right)}\right)-\nabla F\left(\mathbf{A}_{t}^{\left( n \right)}\right)\Big\rangle=0.
\end{align*}
Then we complete the proof.
\end{proof}

\subsection{Proof of Lemma~\ref{lemma_4}}
\begin{proof}
Note that
\begin{align*}
&~~~\mathbb{E}\left[\left\|{g}_{k}(\mathbf{A}_{k,t,t'}^{\left( n \right)})-\nabla{F}_{k}(\mathbf{A}_{k,t}^{\left( n \right)})\right\|^{2}\right]\\
&\leq 2\mathbb{E}\left[\left\|{g}_{k}(\mathbf{A}_{k,t,t'}^{\left( n \right)})-\nabla{F}_{k}(\mathbf{A}_{k,t,t'}^{\left( n \right)})\right\|^{2}\right.\\
&~~~~~~\left.\quad+\left\|\nabla{F}_{k}(\mathbf{A}_{k,t,t'}^{\left( n \right)})-\nabla{F}_{k}(\mathbf{A}_{k,t}^{\left( n \right)})\right\|^{2}\right]\\
&\overset{(a)}{\leq}2\left(\rho^2+L_F^2\mathbb{E}\left[\left\|\mathbf{A}_{k,t,t'}^{\left( n \right)}-\mathbf{A}_{k,t}^{\left( n \right)}\right\|^{2} \right]\right),
\end{align*}
where the first and second inequalities are due to Propositions \ref{Proposition3} and \ref{Proposition2}, respectively. We next bound the drift of local update of client $k$ from global model $\left\|\mathbf{A}_{k,t,t'}^{\left( n \right)}-\mathbf{A}_{k,t}^{\left( n \right)}\right\|^{2}$ as follows
\begin{align*}
&~~~\mathbb{E}\left[\left\|\mathbf{A}_{k,t,t'}^{\left( n \right)}-\mathbf{A}_{k,t}^{\left( n \right)}\right\|^{2} \right]\\
&=\mathbb{E}\left[\left\|\mathbf{A}_{k,t,t'-1}^{\left( n \right)}-\mathbf{A}_{k,t}^{\left( n \right)}-\eta g_{k,t,t'-1}^{n}\right\|^{2} \right] \\
&\leq 2 \mathbb{E}\left[\left\|\mathbf{A}_{k,t,t'-1}^{\left( n \right)}-\mathbf{A}_{k,t}^{\left( n \right)}-\eta \nabla F_{k}\left(\mathbf{A}_{k,t}^{\left( n \right)}\right)\right\|^{2}\right.\\
&\left.\quad+\eta^{2}\left\|g_{k,t,t'-1}^{n}-\nabla F_{k}\left(\mathbf{A}_{k,t}^{\left( n \right)}\right)\right\|^{2}\right] \\
&\leq 2\left(1+\frac{1}{2 \tau}\right) \mathbb{E}\left[\left\|\mathbf{A}_{k,t,t'-1}^{\left( n \right)}-\mathbf{A}_{k,t}^{\left( n \right)}\right\|^{2}\right]\\
&\quad+2(1+2 \tau) \eta^{2} \mathbb{E}\left[\left\|\nabla F_{k}\left(\mathbf{A}_{k,t}^{\left( n \right)}\right)\right\|^{2}\right]\\
&\quad+4 \eta^{2}\left(\rho^2+L_F^{2} \mathbb{E}\left[\left\|\mathbf{A}_{k,t,t'-1}^{\left( n \right)}-\mathbf{A}_{k,t}^{\left( n \right)}\right\|^{2}\right]\right) \\
&=2\left(1+\frac{1}{2 \tau}+2 \eta^{2} L_F^{2}\right) \mathbb{E}\left[\left\|\mathbf{A}_{k,t,t'-1}^{\left( n \right)}-\mathbf{A}_{k,t}^{\left( n \right)}\right\|^{2}\right]\\
&\quad+2(1+2 \tau) \eta^{2} \mathbb{E}\left[\left\|\nabla F_{k}\left(\mathbf{A}_{k,t}^{\left( n \right)}\right)\right\|^{2}\right]+4 \eta^{2} \rho^{2} \\
&\overset{(b)}{\leq} 2\left(1+\frac{1}{\tau}\right) \mathbb{E}\left[\left\|\mathbf{A}_{k,t,t'-1}^{\left( n \right)}-\mathbf{A}_{k,t}^{\left( n \right)}\right\|^{2}\right]\\
&\quad+2(1+2 \tau) \eta^{2} \mathbb{E}\left[\left\|\nabla F_{k}\left(\mathbf{A}_{k,t}^{\left( n \right)}\right)\right\|^{2}\right]+4 \eta^{2} \rho^{2} \\
&\overset{(c)}{\leq}\left(\frac{6 \tilde{\eta}^{2}}{\beta^{2} \tau} \mathbb{E}\left[\left\|\nabla F_{k}\left(\mathbf{A}_{k,t}^{\left( n \right)}\right)\right\|^{2}\right]+\frac{4 \tilde{\eta}^{2}  \rho^{2}}{\beta^{2} \tau^{2}}\right) \sum_{t'=0}^{\tau-1} 2\left(1+\frac{1}{\tau}\right)^{t'} \\
&\overset{(d)}{\leq} \frac{8 \tilde{\eta}^{2}}{\beta^{2}}\left(3 \mathbb{E}\left[\left\|\nabla F_{k}\left(\mathbf{A}_{k,t}^{\left( n \right)}\right)\right\|^{2}\right]+\frac{2  \rho^{2}}{\tau}\right),
\end{align*}
where $(b)$ is by having $2 \eta^{2} L_F^{2}=2 L_F^{2} \frac{\tilde{\eta}^{2}}{\beta^{2} \tau^{2}} \leq \frac{1}{2 \tau^{2}} \leq \frac{1}{2 \tau}$ when $\tilde{\eta}^{2} \leq \frac{\beta^{2}}{4 L_F^{2}}$, for all $\tau \geq 1$. $(c)$ is due to unrolling $(b)$ recursively, and $2(1+2 \tau) \eta^{2}=2(1+2 \tau) \frac{\tilde{\eta}^{2}}{\beta^{2} \tau^{2}} \leq \frac{6 \tilde{\eta}^{2}}{\beta^{2} \tau}$ because $\frac{1+2 \tau}{\tau} \leq 3$ when $\tau \geq 1$. We have $(d)$ because $\sum_{t'=0}^{\tau-1}(1+1 / \tau)^{t'}=\frac{(1+1 / \tau)^{\tau}-1}{1 / \tau} \leq \frac{e-1}{1 / \tau} \leq 2 \tau$, by using the facts that $\sum_{i=0}^{n-1} x^{i}=\frac{x^{n}-1}{x-1}$ and $\left(1+\frac{x}{n}\right)^{n} \leq e^{x}$ for any $x \in \mathbb{R}, n \in \mathbb{N}$. Substituting $(d)$ to $(a)$, we obtain
\begin{align*}
&\mathbb{E}\left[\left\|{\mathbf{g}}_{k}(\mathbf{A}_{k,t,t'}^{\left( n \right)})-\nabla{F}_{k}(\mathbf{A}_{k,t}^{\left( n \right)})\right\|^{2}\right]\\
&\leq 2  \rho^{2}+\frac{16 \tilde{\eta}^{2} L_F^{2}}{\beta^{2}}\left(3 \mathbb{E}\left[\left\|\nabla F_{k}\left(\mathbf{A}_{k,t}^{\left( n \right)}\right)\right\|^{2}\right]+\frac{2  \rho^{2}}{\tau}\right).
\end{align*}
By taking average over $K$ and $\tau$, we finish the proof.
\end{proof}

\subsection{Proof of Theorem~\ref{theorem_1}}
\begin{proof}
We first prove part (a). Due to the $L_{F}$-smoothness of $F(\cdot)$, we have
\begin{align*}
&~~~\mathbb{E}\left[F\left(\mathbf{A}_{t+1}^{\left( n \right)}\right)-F\left(\mathbf{A}_{t}^{\left( n \right)}\right)\right] \\
&\leq \mathbb{E}\left[\left\langle\nabla F\left(\mathbf{A}_{t}^{\left( n \right)}\right), \mathbf{A}_{t+1}^{\left( n \right)}-\mathbf{A}_{t}^{\left( n \right)}\right\rangle\right]\\
&\quad+\frac{L_{F}}{2} \mathbb{E}\left[\left\|\mathbf{A}_{t+1}^{\left( n \right)}-\mathbf{A}_{t}^{\left( n \right)}\right\|^{2}\right] \\
&=-\tilde{\eta} \mathbb{E}\left[\left\langle\nabla F\left(\mathbf{A}_{t}^{\left( n \right)}\right), g_{t}\right\rangle\right]+\frac{\tilde{\eta}^{2} L_{F}}{2} \mathbb{E}\left[\left\|g_{t}\right\|^{2}\right] \\
&=-\tilde{\eta} \mathbb{E}\left[\left\|\nabla F\left(\mathbf{A}_{t}^{\left( n \right)}\right)\right\|^{2}\right]\\
&\quad-\tilde{\eta} \mathbb{E}\left[\left\langle\nabla F\left(\mathbf{A}_{t}^{\left( n \right)}\right), g_{t}-\nabla F\left(\mathbf{A}_{t}^{\left( n \right)}\right)\right\rangle\right]+\frac{\tilde{\eta}^{2} L_{F}}{2} \mathbb{E}\left[\left\|g_{t}\right\|^{2}\right] \\
&\overset{(a)}\leq-\tilde{\eta} \mathbb{E}\left[\left\|\nabla F\left(\mathbf{A}_{t}^{\left( n \right)}\right)\right\|^{2}\right]+\frac{\tilde{\eta}}{2} \mathbb{E}\left[\left\|\nabla F\left(\mathbf{A}_{t}^{\left( n \right)}\right)\right\|^{2}\right]\\
&\quad+\frac{\tilde{\eta}}{2} \mathbb{E}\left\|\sum_{k,t'}^{K,\tau}\frac{|\mathcal{D}_k|}{|\mathcal{D}|\tau} g_{k,t,t'}-\nabla F_{k}\left(\mathbf{A}_{k,t}^{\left( n \right)}\right)\right\|^{2}+\frac{\tilde{\eta}^{2} L_{F}}{2} \mathbb{E}\left[\left\|g_{t}\right\|^{2}\right] \\
&\overset{(b)}\leq-\frac{\tilde{\eta}}{2} \mathbb{E}\left[\left\|\nabla F\left(\mathbf{A}_{t}^{\left( n \right)}\right)\right\|^{2}\right]\\
&\quad+\frac{3 L_{F} \tilde{\eta}^{2}}{2} \mathbb{E}\left\|\frac{1}{S} \sum_{k \in \mathcal{S}^{t}} \nabla F_{k}\left(\mathbf{A}_{k,t}^{\left( n \right)}\right)-\nabla F\left(\mathbf{A}_{t}^{\left( n \right)}\right)\right\|^{2} \\
&\quad+\frac{\tilde{\eta}\left(1+3 L_{F} \tilde{\eta}\right)}{2} \sum_{k,t'}^{K,\tau}\frac{|\mathcal{D}_k|}{|\mathcal{D}|\tau} \mathbb{E}\left[\left\|g_{k,t,t'}-\nabla F_{k}\left(\mathbf{A}_{k,t}^{\left( n \right)}\right)\right\|^{2}\right]\\
&\quad+\frac{3 \tilde{\eta}^{2} L_{F}}{2} \mathbb{E}\left[\left\|\nabla F\left(\mathbf{A}_{t}^{\left( n \right)}\right)\right\|^{2}\right] \\
&\overset{(c)}\leq-\frac{\tilde{\eta}(1-3L_F\tilde{\eta})}{2}\mathbb{E}\left[\left\|\nabla F\left(\mathbf{A}_{t}^{\left( n \right)}\right)\right\|^{2}\right]\\
&\quad+\frac{3 L_{F} \tilde{\eta}^{2}}{2}\sum_{k=1}^{K}\frac{|\mathcal{D}| / S-|\mathcal{D}_k|}{|\mathcal{D}|-|\mathcal{D}_k|}  \frac{|\mathcal{D}_k|}{|\mathcal{D}|}\mathbb{E}\left\|\nabla F_{k}\left(\mathbf{A}_{k,t}^{\left( n \right)}\right)\right.\\
&\left.\quad-\nabla F\left(\mathbf{A}_{t}^{\left( n \right)}\right)\right\|^{2}\\
&\quad+\frac{\tilde{\eta}\left(1+3 L_{F} \tilde{\eta}\right)}{2}\left[2\rho^{2}+\frac{16 \tilde{\eta}^{2} L_{F}^{2}}{\beta^{2}}\left(\frac{2\rho^{2}}{\tau}\right.\right.\\
&\left.\left.\quad+3 \sum_{k=1}^{K} \frac{|\mathcal{D}_k|}{|\mathcal{D}|} \mathbb{E}\left[\left\|\nabla F_{k}\left(\mathbf{A}_{k,t}^{\left( n \right)}\right)-\nabla F\left(\mathbf{A}_{t}^{\left( n \right)}\right)\right\|^{2}\right]\right.\right.\\
&\left.\left.\quad
+3 \mathbb{E}\left[\left\|\nabla F\left(\mathbf{A}_{t}^{\left( n \right)}\right)\right\|^{2}\right]\right)\right]\\
&\overset{(d)}\leq-\frac{\tilde{\eta}\left(1-3 L_{F} \tilde{\eta}\right)}{2} \mathbb{E}\left[\left\|\nabla F\left(\mathbf{A}_{t}^{\left( n \right)}\right)\right\|^{2}\right]\\
&\quad+\frac{3 L_{F} \tilde{\eta}^{2}}{2} \frac{C}{J}\left(2\sigma_f^{2}+8\nu \right)\\
&\quad+\frac{\tilde{\eta}\left(1+3 L_{F} \tilde{\eta}\right)}{2}\left[2 \rho^{2}+\frac{16 \tilde{\eta}^{2} L_{F}^{2}}{\beta^{2}}\left(\frac{2 \rho^{2}}{\tau}\right.\right.\\
&\left.\left.\quad+6 \sigma_f^{2}+24\nu+ 3\mathbb{E}\left[\left\|\nabla F\left(\mathbf{A}_{t}^{\left( n \right)}\right)\right\|^{2}\right]\right)\right]\\
&\overset{(e)}=-\frac{\bar{\eta}\left(1-3 L_{F} \tilde{\eta}\right)}{2} \mathbb{E}\left[\left\|\nabla F\left(\mathbf{A}_{t}^{\left( n \right)}\right)\right\|^{2}\right]\\
&\quad+\frac{24 \tilde{\eta}^3L_{F}^2\left(1+3 L_{F} \tilde{\eta}\right) }{\beta^{2}} \mathbb{E}\left[\left\|\nabla F\left(\mathbf{A}_{t}^{\left( n \right)}\right)\right\|^{2}\right]\\
&\quad+\frac{\tilde{\eta}^{3}}{\beta^{2}}\left(1+3 L_{F} \tilde{\eta}\right)16L_{F}^2 \left(\frac{\rho^2}{\tau}+3 \sigma_{f}^{2}+12\nu\right)\\
&\quad+\tilde{\eta}^{2}3 L_{F} \frac{C}{J} \left(\sigma_{f}^{2}+4\nu\right)+\tilde{\eta}\left(1+3 L_{F} \tilde{\eta}\right) \rho^{2}\\
&\overset{(f)}\leq-\tilde{\eta}(\underbrace{\frac{1}{2}-\frac{75}{2}\tilde{\eta} L_{F}}_{\geq 1 / 4 \text { when } \bar{\eta} \text { satisfied } \bar{\eta} \leq \frac{1}{150L_F}}  ) \mathbb{E}\left[\left\|\nabla F\left(\mathbf{A}_{t}^{\left( n \right)}\right)\right\|^{2}\right]\\
&\quad+\frac{\tilde{\eta}^{3}}{\beta^{2}}\left(1+3 L_{F} \tilde{\eta}\right)16L_{F}^2 \left(\frac{\rho^2}{\tau}+3 \sigma_{f}^{2}+12\nu\right)\\
&\quad+\tilde{\eta}^{2}3 L_{F} \frac{C}{J} \left(\sigma_{f}^{2}+4\nu\right)+\tilde{\eta}\left(1+3 L_{F} \tilde{\eta}\right) \rho^{2}\\
&\overset{(g)}\leq-\frac{\bar{\eta}}{4}\left\|\nabla F\left(\mathbf{A}_{t}^{\left( n \right)}\right)\right\|^{2}+\frac{\tilde{\eta}^{3}}{\beta^{2}} \underbrace{32L_F^2\left(\frac{\rho^2}{\tau}+3 \sigma_{f}^{2}+12\nu\right)}_{=: C_{1}}\\
&\quad+\tilde{\eta}^{2} \underbrace{3 L_{F} \frac{C}{J} \left(\sigma_{f}^{2}+4\nu\right)}_{=: C_{2}}+\bar{\eta} \underbrace{2  \rho^{2}}_{=: C_{3}},
\end{align*}
where $(a)$ is due to Cauchy-Swartz and AM-GM inequalities, $(b)$ is by decomposing $\left\|g_t\right\|^2$ into three terms, which is similar to \cite{3-2020Personalized}, and $(c)$ is by using Lemmas \ref{lemma_3} and \ref{lemma_4}, and the fact that $\mathbb{E}[\left\|X\right\|^{2}]=\mathbb{E}[\left\|X-\mathbb{E}\left[X\right]\right\|^{2}]+\mathbb{E}[\left\|X\right\|]^{2}$ for any vertor of random variable $X$. We have $(d)$ by Lemma \ref{lemma_2} and $\sum_{k=1}^{K}\frac{|\mathcal{D}| / S-|\mathcal{D}_k|}{|\mathcal{D}|-|\mathcal{D}_k|}=\frac{C}{J}$ with $k=1,\dots,K$, $(e)$ by re-arranging the terms, and $(f)$ by having $1+3L_F\tilde{\eta}\leq1+\frac{3\beta}{2}\leq3\beta$ when $\tilde{\eta}\leq\frac{\beta}{2L_F}$ according to Lemma \ref{lemma_4} and $\beta\geq1$. Finally, we have $(g)$ by using $1+3L_F\tilde{\eta}\leq2$.

By re-arranging the terms of $(g)$ and telescoping, we have
\begin{align*}
&~~~\frac{1}{4T}\sum_{t=0}^{T-1}\mathbb{E}\left[\left\|\nabla F\left(\mathbf{A}_{t}^{\left( n \right)}\right)\right\|^{2}\right]\\
&\leq\frac{1}{\tilde{\eta}T} \mathbb{E}\left[F\left(\mathbf{A}_{0}^{\left( n \right)}\right)-F\left(\mathbf{A}_{T}^{\left( n \right)}\right)\right] +\frac{\tilde{\eta}^2}{\beta^{2}}C_1+\tilde{\eta}C_2+C_3.
\end{align*}
We hence have
\begin{align*}
\frac{1}{T}\sum_{t=0}^{T-1}\mathbb{E}\left[\left\|\nabla F\left(\mathbf{A}_{t}^{\left( n \right)}\right)\right\|^{2}\right] \leq4\left(\frac{\Delta_F}{\tilde{\eta}T}+\frac{\tilde{\eta}^2}{\beta^{2}}C_1+\tilde{\eta}C_2+C_3\right),
\end{align*}
where $\Delta_F\triangleq F(\mathbf{A}_{0}^{\left( n \right)})-F(\hat{\mathbf{A}}^{\left( n \right)})$.

We next prove part (b) as follows
\begin{align*}
&~~~\sum_{k=1}^{K} \frac{|\mathcal{D}_k|}{|\mathcal{D}|} \mathbb{E}\left[\left\|\tilde{\theta}_{k,t}\left(\mathbf{A}_{k,t}^{\left( n \right)}\right)-\mathbf{A}_{t}^{\left( n \right)}\right\|^{2}\right]\\
&\overset{(a)}\leq \sum_{k=1}^{K} \frac{|\mathcal{D}_k|}{|\mathcal{D}|} 2 \mathbb{E}\left[\left\|\tilde{\theta}_{k,t}\left(\mathbf{A}_{k,t}^{\left( n \right)}\right)-\hat{\theta}_{k,t}\right\|^{2}\right.\\
&~\quad\left.+\left\|\hat{\theta}_{k,t}\left(\mathbf{A}_{t}^{\left( n \right)}\right)-\mathbf{A}_{t}^{\left( n \right)}\right\|^{2}\right] \\
& \overset{(b)}\leq  \frac{4}{(\lambda-L)^{2}}\left(\frac{\gamma_{f}^{2}}{|\mathcal{B}|}+\nu\right)\\
&\quad+2 \sum_{k=1}^{K} \frac{|\mathcal{D}_k|}{|\mathcal{D}|}\frac{1}{\mu_F} \mathbb{E}\left[\left\|\nabla F_{k}\left(\hat{\theta}_{k,t}\left(\mathbf{A}_{t}^{\left( n \right)}\right)\right)-\nabla F_{k}\left(\mathbf{A}_{t}^{\left( n \right)}\right)\right\|^{2}\right] \\
& \overset{(c)}\leq  \frac{4}{(\lambda-L)^{2}}\left(\frac{\gamma_{f}^{2}}{|\mathcal{B}|}+\nu\right)\\
&\quad+\frac{4}{\mu_F} \sum_{k=1}^{K} \frac{|\mathcal{D}_k|}{|\mathcal{D}|} \mathbb{E}\left[\left\|\nabla F_{k}\left(\hat{\theta}_{k,t}\left(\mathbf{A}_{t}^{\left( n \right)}\right)\right)\right\|^{2}\right.\\
&\left.\quad+\left\|\nabla F_{k}\left(\mathbf{A}_{t}^{\left( n \right)}\right)\right\|^{2}\right]\\
&\overset{(d)}\leq  \frac{4}{(\lambda-L)^{2}}\left(\frac{\gamma_{f}^{2}}{|\mathcal{B}|}+\nu\right)+\frac{8\rho}{\mu_F}\\
&\quad+\frac{4}{\mu_F}\sum_{k=1}^{K} \frac{|\mathcal{D}_k|}{|\mathcal{D}|}\mathbb{E}\left[\left\|\nabla F_{k}\left(\mathbf{A}_{t}^{\left( n \right)}\right)\right\|^{2}\right]\\
&\overset{(e)}\leq  \frac{4}{(\lambda-L)^{2}}\left(\frac{\gamma_{f}^{2}}{|\mathcal{B}|}+\nu\right)+\frac{8\rho}{\mu_F}\\
&\quad+\frac{4}{\mu_F}\sum_{k=1}^{K} \frac{|\mathcal{D}_k|}{|\mathcal{D}|}\left(\mathbb{E}\left[\left\|\nabla F_{k}\left(\mathbf{A}_{t}^{\left( n \right)}\right)-\nabla F\left(\mathbf{A}_{t}^{\left( n \right)}\right)\right\|^{2}\right]\right.\\
&\quad\left.+\mathbb{E}\left[\left\|\nabla F\left(\mathbf{A}_{t}^{\left( n \right)}\right)\right\|^{2}\right]\right)\\
&\overset{(f)}\leq  \frac{4}{(\lambda-L)^{2}}\left(\frac{\gamma_{f}^{2}}{|\mathcal{B}|}+\nu\right)+\frac{4}{\mu_F}(2\rho+2 \sigma_{f}^{2}+8\nu)\\
&\quad+\mathbb{E}\left[\left\|\nabla F\left(\mathbf{A}_{t}^{\left( n \right)}\right)\right\|^{2}\right],
\end{align*}
where $(a)$ is due to Proposition \ref{Proposition3}, $(b)$ is due to the Lemma \ref{lemma_1} and Proposition
\ref{Proposition2}, $(c)$ and $(d)$ by having $\frac{1}{2}\mathbb{E}[\left\|\nabla F_{k}\right\|^{2}]\leq\frac{1}{2}\mathbb{E}[\left\|\nabla g_{k}+\nabla F_{k}\right\|^{2}]\leq\mathbb{E}[\left\|\nabla g_{k}-\nabla F_{k}\right\|^{2}]\leq\rho$ according to Assumption \ref{assumption_4}, $(f)$ is due to Lemma \ref{lemma_2}.

Summing the above from $t=0$ to $T$, we have
\begin{align*}
&\frac{1}{T}\sum_{t=0}^{T-1}\sum_{k=1}^{K} \frac{|\mathcal{D}_k|}{|\mathcal{D}|} \mathbb{E}\left[\left\|\tilde{\theta}_{k,t}\left(\mathbf{A}_{k,t}^{\left( n \right)}\right)-\mathbf{A}_{t}^{\left( n \right)}\right\|^{2}\right]\\
&\quad\leq\frac{1}{T}\sum_{t=0}^{T-1}\mathbb{E}\left[\left\|\nabla F\left(\mathbf{A}_{t}^{\left( n \right)}\right)\right\|^{2}\right]\\
&~~~\quad+\frac{4}{(\lambda-L)^{2}}\left(\frac{\gamma_{f}^{2}}{|\mathcal{B}|}+\nu\right)+\frac{4}{\mu_F}(2\rho+2 \sigma_{f}^{2}+8\nu),
\end{align*}
then we finish the proof.
\end{proof}


\bibliographystyle{IEEEtran}
\bibliography{TDPFed.bib}

\end{document}